\documentclass{article} \usepackage{iclr2027_conference,times}

\usepackage{xcolor}
\usepackage{hyperref}
\usepackage{url}
\usepackage{amsmath, amssymb, amsthm}
\usepackage{mathtools}
\usepackage{booktabs}
\usepackage{hyperref}
\usepackage{enumitem}
\usepackage{multirow}
\usepackage[dvipsnames]{xcolor}
\usepackage{microtype}
\usepackage{researchPack}
\usepackage{natbib}
\usepackage{cleveref}
\usepackage{tabularx}
\usepackage{multirow}
\usepackage{nicefrac}
\hypersetup{
  colorlinks=true,
  linkcolor=blue!50!black,
  citecolor=blue!50!black,
  urlcolor=blue!50!black
}
\newtheorem*{lemma:flip-decomp-restated}{Lemma 1}
\newtheorem*{lemma:label-aligned-flip-rates}{Lemma 2}
\newtheorem*{proposition:hoeffding}{Proposition 2}
\newtheorem*{proposition:finite_multi}{Proposition 1}
\newtheorem*{proposition:group_recover_not_necessary}{Proposition 1}
\newtheorem*{theorem:probe_orth}{Theorem 1}
\newtheorem*{theorem:residual_alignment}{Theorem 1}
\newtheorem*{corollary:explicit_bias}{Corollary 1}
\newtheorem*{theorem:acc_improv}{Theorem 2}
\newtheorem*{definition:mce}{Definition 3}
\newtheorem*{corollary:group_recoverability_suffices}{Corollary 1}
\newtheorem{corollary:group_recoverability_suffices_app}{Corollary}

\usepackage{xspace}
\newcommand{\indicator}[1]{\mathbb{I}_{\{#1\}}}

\newcommand{\mbx}{\mathbf{x}}
\newcommand{\pup}{\bbP(X \in \calC^\uparrow)}
\newcommand{\pdn}{\bbP(X \in \calC^\downarrow)}
\newcommand{\qdn}{\bbP(Y = 1 \mid X \in \calC^\downarrow)}
\newcommand{\qup}{\bbP(Y = 1 \mid X \in \calC^\uparrow)}

\newcommand{\method}{{\tt CaLLM}\xspace}
\newcommand{\methodG}{{\tt CaLLM}\xspace$\!\!+g$}

\newcommand{\etaxp}{\eta_{X}}

\newcommand{\etax}{\bbE[Y\mid X]}

\newcommand{\mcgrad}{{\tt MCGrad}\xspace}

\title{A CaLLM Judge is a Better Judge: \\ Mitigating Bias in LLM-as-a-judge}

\author{Cesare Barbera$^{1,4}$\thanks{Joint first authorship and corresponding authors.} \quad
Luca Stradiotti$^{2, \ast}$ \quad
Lorenzo Cascioli$^{2}$ \quad
Lorenzo Perini$^{3}$ \\
\textbf{Niek Tax}$^{3}$ \quad
\textbf{Daniel Haimovich}$^{3}$ \quad
\textbf{Jesse Davis}$^{2}$ \quad
\textbf{Andrea Passerini}$^{4}$ \quad
\textbf{Andrea Pugnana}$^{4}$ \\[4pt]
$^{1}$University of Pisa \quad
$^{2}$KU Leuven \quad
$^{3}$Meta Platforms, Inc. \quad
$^{4}$University of Trento\\[4pt]
\texttt{cesare.barbera@phd.unipi.it}, \quad
\texttt{\{name.surname\}@kuleuven.be}, \\
\texttt{\{lorenzoperini,niek,danielha\}@meta.com}, \quad
\texttt{\{name.surname\}@unitn.it}
}

\iclrfinalcopy \begin{document}

\maketitle
\begin{abstract}
Large language models (LLMs) are increasingly used as judges to predict which of two candidate answers a human evaluator would prefer, providing a scalable alternative when human evaluation is too costly. Unfortunately, LLM judgments can exhibit systematic biases. Prior work has documented position, verbosity or model family of the candidate answers influencing the judge decision rather than their quality, while additional biases may remain undocumented. 
\emph{Multicalibration} provides a natural framework to address this problem: each bias corresponds to miscalibration on an identifiable group, \eg, family bias manifests as overconfidence in answers generated by models belonging to the same family of the judge. The challenge is to fix miscalibration simultaneously across many, potentially unknown groups. 
To this end, we introduce \method, a post hoc framework that multicalibrates an LLM judge, thereby bounding all such biases simultaneously. We prove that \method bounds the bias of every group whose error structure is captured by an available representation, and that augmenting it with explicit bias groups further reduces the residual bias. We then derive theoretical conditions under which \method improves the judge's accuracy, rather than trading it for calibration. Finally, we evaluate \method across pairwise preference benchmarks, showing that it reduces the multicalibration error on both specified and unspecified bias groups, while improving predictive performance of the judge.
\end{abstract}

\section{Introduction}

When evaluating the outputs of generative models, obtaining human annotations at scale is often too slow and expensive~\citep{li2024survey}. As a consequence, the LLM-as-a-judge framework has emerged as a practical alternative~\citep{zheng2023judging}: a strong language model is prompted to score or compare candidate responses, enabling evaluations at a scale that human raters cannot match.
However, LLM judges can exhibit systematic biases that undermine the reliability of their evaluations~\citep{gu2024survey,DBLP:journals/corr/abs-2508-06225}. For example, \citet{zheng2023judging} identify three such biases: \emph{position bias}, \ie the judge favours the answer shown first or second independently of its content, \emph{verbosity bias}, \ie the judge favours longer answers regardless of quality, and \emph{family bias}, \ie the judge favours answers produced by its own model family.

Although mitigation strategies have been proposed, controlling these biases with theoretical guarantees remains an open challenge~\citep{DBLP:journals/corr/abs-2412-12509, li2025generation}. 
We propose viewing systematic biases through the lens of group-wise miscalibration. From this perspective, each bias is a miscalibration of the judge's confidence within a particular group: 
\eg, \textit{verbosity bias} arises in pairwise evaluations where one answer is much longer than the other. 
To address this within-group miscalibration problem, we resort to \textit{Multicalibration}~\citep{hebert2018multicalibration}, a theoretical framework that requires predictions to be calibrated simultaneously across identifiable groups. Specifically, for each group, the absolute expected discrepancy between the predictions and the label must be below a small threshold~$\alpha$.
Problematically, the bias groups of interest are defined over the original input space, whereas LLM evaluation operates on unstructured text and a multicalibrator can only observe a learned representation of it. This raises the question of whether multicalibration in representation space controls biases defined in the original input space.

In this work, we address these challenges with Multi-\underline{Ca}librated \underline{L}arge \underline{L}anguage \underline{M}odels (\method), a post hoc bias-mitigation framework that implements bias-correction guarantees in practice (see \Cref{fig:vis_abstract}). First, we cast the LLM-as-a-judge bias problem within a multicalibration framework, mapping biases to groups.
We then prove that multicalibration in representation space controls bias groups in the original input space whenever their error structure is recoverable from that representation. Importantly, this does not require the groups to be explicitly specified to the calibrator.
Finally, we show that, under suitable conditions, bias reduction need not come at the expense of predictive performance, rather \method guarantees improvements of the judge's accuracy.
Overall, \method turns these calibration and performance guarantees into a practical post-hoc procedure for mitigating both explicitly specified and recoverable biases in LLM judges.

\paragraph{Our contributions.} In a nutshell, our contributions are: 
\begin{itemize}
    \item [$(i)$] We formalize the LLM-as-a-Judge biases as a specific instance of a multicalibration problem;
    \item [$(ii)$] We propose \method, a post-hoc bias-mitigation strategy that leverages multicalibration to bound identifiable judge's bias;
\item [$(iii)$] We prove that multicalibration in representation space bounds recoverable input-space bias patterns and derive conditions under which it improves the judge’s accuracy;
    \item [$(iv)$] We show on five benchmark datasets that \method mitigates documented biases, while simultaneously improving the performance of the LLM-as-a-judge framework.
\end{itemize}

\begin{figure}
    \centering
    \includegraphics[width=.9\linewidth]{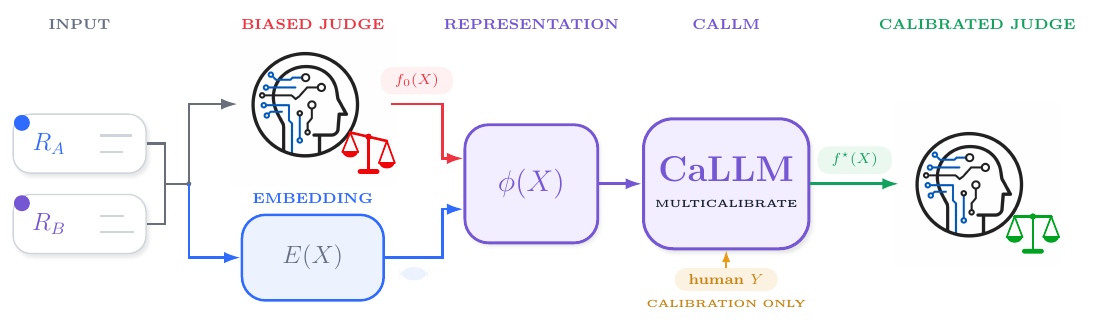}
    \caption{\method's pipeline. We feed \method with some representation $\phi(X)$ of the input (e.g., embeddings $E(X)$ and base scores $f_0(X)$). The final output is a multi-calibrated score $f^\star(X)$. }
    \label{fig:vis_abstract}
\end{figure}

\section{Preliminaries}

\textbf{LLM-as-a-judge.}
We focus on pairwise evaluation of candidate answers, which has emerged as a popular approach for assessing the quality of LLM outputs~\citep{gu2024survey}. In this setting, the judge receives a user prompt with two anonymized candidate responses and selects the best output, emulating human preference. 
More formally, a pairwise evaluation instance is a random variable $X = (Q, R_1, R_2) \in \mathcal X$, where $Q$ is the prompt 
and $R_1, R_2$ are two candidate responses, respectively.
The binary label space $\calY \in \{0,1\}$ indicates human preference, with $Y=1$ ($Y=0$) denoting that the human prefers the first (second) answer.
We assume access to a given dataset \( \calD = \{(\mbx_i, y_i)\}_{i=1}^n \) of input-output pairs drawn from an unknown joint distribution \( \mathcal{P} \) over \( \mathcal{X} \times \mathcal{Y} \). The LLM judge induces a continuous scoring function $f : \calX \to [0,1]$, which estimates the true conditional probability $\bbP(Y=1 \mid X)$. Thus, $f(X)$ is the judge's confidence in the first answer. 
We distinguish two instances of the judge's confidence: $f_0$ denotes the \emph{base} score returned by the judge, and $f^\star$ the calibrated score. We write $f$ whenever a statement holds for both. We obtain hard predictions by thresholding $f$ at $\nicefrac{1}{2}$ and define the judge's accuracy as $\textstyle \mathrm{Acc}(f) = \bbP\bigl(\indicator{f(X)\ge\nicefrac{1}{2}} = Y\bigr)$.

\textbf{Bias groups.}
\label{sec:bias_groups}
We represent each LLM bias through a collection of binary groups $\calG$.
Due to the inherently comparative nature of the task, to define groups we first impose a fixed ordering on the candidate responses.
Let
\(
\pi:\calX\to\widetilde{\calX}
\)
denote a deterministic reordering map such that
\(
\pi(X)=(Q,R_A,R_B),
\)
where $R_A$ and $R_B$ are the original responses ordered according to a rule
independent of $Y$ and $f$\footnote{This ensures that group membership remains
unchanged after calibration.}, e.g., the lexicographical ordering of their
source-model names.
We represent group membership through a measurable attribute 
\(
a:\widetilde{\calX}\to\calV_a,
\)
whose value indicates how $R_A$ compares with $R_B$ along the corresponding
attribute dimension. Each value $v\in\calV_a$ induces a binary group
\(
\tilde g_{a,v}(\tilde X)
:=
\indicator{a(\tilde X)=v}.
\)
The corresponding group indicator on the original input space is defined by
composition,
\(
g_{a,v}
:=
\tilde g_{a,v}\circ\pi
\) 
so that 
\(
g_{a,v}(X)
=
\tilde g_{a,v}(\pi(X)).
\)
For simplicity, we henceforth write a generic group as $g$, omitting the indices $(a,v)$ when they are not needed explicitly.
\begin{example*}[Documented-biases groups] 
Documented biases~\citep{zheng2023judging} can be represented through the following attributes:
(i) \textbf{Position}, $a^{\mathrm{pos}}(\tilde X)\in\{-1,+1\}$, indicating whether $R_A$ is shown second or first;
(ii) \textbf{Verbosity}, $a^{\mathrm{len}}(\tilde X)\in\{-1,0,+1\}$, indicating whether $R_A$ is shorter, the same or longer than $R_B$;
(iii) \textbf{Family}, $a^{\mathrm{fam}}(\tilde X)\in\{-1,0,+1, +2\}$, indicating whether neither response belongs to the judge's model family, only $R_B$  does, only $R_A$ does, or both do.
Each group is defined by the presence of a specific attribute value, or by a combination of values across multiple attributes.
\label{remark:biasaxes}
\end{example*}
A \textit{bias} is then a systematic discrepancy between the judge's confidence and the true human preference rate \emph{within} these groups. 
However, measuring bias without accounting for the judge's confidence may hide discrimination that emerges only at particular confidence levels. Thus, we additionally condition on a finite collection of confidence bands $\calI$, \eg, sub-intervals of $[0,1]$. Then, for each group, we compute the conditional error within each confidence band and define the judge's per-group bias as the largest of these errors~\citep{DBLP:conf/kdd/TaxPLHKOVA26}:
\begin{definition}[Per-group bias]
\label{def:peraxisbias}
Let $\mathcal{G}$ be a collection of measurable group indicators such that $g:\calX\to\{0,1\}$ 
and let $\calI$ be a finite collection of sub-intervals of $[0,1]$. The per-group bias of $f$ is the worst conditional deviation across confidence bands:
\[
  \mathrm{bias}_{g,\calI}(f)
  \;:=\;
  \max_{I \in \calI}
  \;\Biggl|\,
  \bbE\Bigl[\,Y - f(X) \,\bigm|\, g(X)= 1,\; f(X) \in I \,\Bigr]
  \,\Biggr| \notag
\]
\end{definition}
We stress that $(i)$ this is a \emph{distributional} notion of bias rather than a bias-specific criterion such as counterfactual swap consistency~\citep{DBLP:conf/ijcnlp/ShiMLDMV25}; $(ii)$ it applies uniformly to both documented and unspecified biases as it measures conditional deviations over arbitrary bias-induced groups.

\textbf{Multicalibration.}
According to~\cref{def:peraxisbias} the judge is unbiased along $g$ when its residuals $Y-f(X)$ average to zero within that group across every confidence band.
This is a within-group calibration constraint, and controlling every group at once maps precisely to \emph{multicalibration}~\citep{hebert2018multicalibration}, which requires $f$ to satisfy the same approximate calibration guarantee simultaneously over a collection $\calG$ of identifiable groups.
\begin{definition}[Multicalibration]
\label{def:multicalibration}
Let $\mathcal{G}$ be a collection of measurable group indicators such that
$g:\calX\to\{0,1\}$ and let $\calI$ be a finite collection of sub-intervals of
$[0,1]$. For $g\in\mathcal{G}$ and $I\in\calI$, define the
multicalibration deviation and the noise scale as
\(
\Delta_{g,I}(f)
\;:=\;
\bbE\bigl[\,\bigl(Y-f(X)\bigr)\,g(X)\,\indicator{f(X)\in I}\,\bigr],
\)
\(
\tau_g(f)^2 \;:=\; \bbE\bigl[\,g(X)\,f(X)\bigl(1-f(X)\bigr)\,\bigr].
\)
Given $\alpha \ge 0,\,$the predictor $f$ is
\emph{$(\mathcal{G},\calI,\alpha)$-multicalibrated} if 
\[
\bigl|\Delta_{g,I}(f)\bigr|
\;\leq\;
\alpha\,\tau_g(f),
\qquad
\forall\,g\in\mathcal{G},\ \forall\,I\in\calI \notag
\]
\end{definition}
\noindent
In other terms, $\Delta_{g,I}$ is the expected amount by which the judge over- or under-states its confidence on the instances such that $g(\mbx) =1 $ and $f(\mbx) \in I$, while $\tau_g$ captures the corresponding scale of sampling variability.
We quantify how far a judge is from this requirement through the Multicalibration Error (MCE)~\citep{guy2025measuring} \ie, the largest group-level Expected Cumulative Calibration Error range (ECCE-R)~\citep{tygert2023ecce}. ECCE-R orders examples by predicted score, cumulatively sums the residuals \(y_i-f(\mbx_i)\), and measures the range spanned by this cumulative error curve.
\begin{definition}[Multicalibration Error]
\label{def:mce}
Consider a collection $\mathcal{G}$ of measurable group indicators $g:\mathcal{X}\rightarrow\{0,1\}$. For each $g\in\mathcal{G}$, let
\(
\mathcal{D}_g
:=
\{(\mbx_i,y_i)\}_{i=1}^{n_g}
\)
denote the sample restricted to instances satisfying $g(\mbx_i)=1$,
and let $\operatorname{ECCE\text{-}R}_g$ be the ECCE-R computed
restricting the sample to $\mathcal{D}_g$. Under the null hypothesis of perfect calibration, define the corresponding noise scale as 
\(
\tau_g(f)
:=
\frac{1}{n_g}
\sqrt{
\sum_{i=1}^{n_g}
f(x_i)\bigl(1-f(x_i)\bigr)
}.
\)
The empirical Multicalibration Error of $f$ over
$\mathcal{G}$ is 
\[
\operatorname{MCE}
:=
\max_{g\in\mathcal{G}}
\frac{
\operatorname{ECCE\text{-}R}_g(f)
}{
\tau_g(f)
} \notag
\]
\end{definition}
We report MCE both relative to group prevalence (\(\mathrm{MCE}_{\%}\)) and in units of the sampling variation expected under perfect calibration (\(\mathrm{MCE}_{\sigma}\)). We also report accuracy together with log loss, which is a strictly proper scoring rule~\citep{gneiting2007strictly}. We provide further details in~\cref{app:metrics}.

\section{Multicalibrated Large Language Models}
\label{sec:method}
Our work tackles the following problem:

\begin{tabularx}{\linewidth}{@{}lX@{}}
    \textbf{Given:} & a calibration set $\mathcal D^{\mathrm{cal}} \subseteq \mathcal D$ and, for each pairwise evaluation, a judge confidence score $f_0(X)$ that $R_1$ is preferred and a representation $\phi:\calX\to\mathbb{R}^d$; \\[0.3em]
    
    \textbf{Do:} & Multicalibrate $f_0(X)$, producing a confidence score $f^\star(X)$  that controls the bias of every $g\in\calG$ whose error structure is recoverable from $\phi(X)$, \ie, $f^\star$ is $(\calG,\calI,\alpha)$ multicalibrated.
\end{tabularx}

The core idea of \method is to start with uncalibrated confidence scores and apply a post-hoc multicalibration procedure to uncover and correct systematic biases. This has three challenges.
\textbf{First}, standard multicalibration is typically formulated over features from which relevant groups can be directly identified, whereas we operate on a latent representation in which bias groups may be only implicitly encoded. We show that multicalibration nevertheless bounds the residual bias of any group whose error structure can be recovered from this representation (\cref{subsec:bias_gaurantees}).
\textbf{Second}, some bias groups may be known a priori, yet relying solely on the implicit encoding in the representation of their error patterns can weaken the bias-mitigation guarantee. We show that providing these groups directly to the multicalibrator can further reduce their residual error (\cref{subsec:named_groups}).
\textbf{Third}, correcting bias can change the judge's decisions, raising the concern that it may reduce accuracy. We show that accuracy improves whenever the confidence gained on corrected decisions is large enough to outweigh the residual calibration error permitted by approximate multicalibration (\cref{subsec:perf_guarantees}).

\subsection{Bias Mitigation Guarantees for Multicalibrated LLM Judges}
\label{subsec:bias_gaurantees}
Classical multicalibration is typically formulated over an input space in which relevant groups can be directly identified through interpretable covariates. In our setting, however, the inputs are
unstructured natural-language responses, and the multicalibrator instead operates on a learned representation $\phi(X) \in \mathbb{R}^d$, which can consist, \eg, of embeddings $E(X)$ and optionally base scores $f_0(X)$. Consequently, the representation need not preserve group membership itself, but may still preserve the systematic error structure associated with that group.
This creates a gap between the groups on which biases are defined and the feature space available to the multicalibrator. 
We therefore characterize how the information preserved by $\phi(X)$ interacts with the judge's systematic prediction error. To this end, we derive an interpretable decomposition of the residual bias in an arbitrary input-space group $g$.
Consider a class $\calH$ of functions $h:\mathbb R^d\to[0,1]$ defined over the representation \(\phi(X)\), which are available to the multicalibrator. Depending on the choice of \(h\), these functions may capture group membership, systematic residual structure, or other identifiable patterns in the representation.
Unlike prior work on proxy-group multicalibration~\citep{DBLP:conf/icml/BhartiCYS25}, our decomposition separates missing label information from missing group information, so a group can be protected when the representation captures its error structure, even if it cannot recover the group. Formally,
let us denote with $B_I = \{f(X) \in I\}$ confidence-band membership and let us define the following quantities \ie,
\begin{align}    
&\lambda_I^2=\bbE\!\left[
\left(
\etax-f(X)
-
\bbE\!\left[
\etax-f(X)
\mid
\phi(X),\,B_I
\right]
\right)^2
\mid B_I
\right], \notag\\
&\varphi_I^2=\bbE\!\left[
\left(
g(X)
-
\bbE\!\left[
g(X)
\mid
\phi(X),\,B_I
\right]
\right)^2
\mid B_I
\right], \notag\\
&\psi_I^2=\bbE\!\left[
\left(
\bbE\!\left[
g(X) -
h(\phi(X))
\mid
\phi(X),\,B_I
\right]
\right)^2
\mid B_I
\right] \notag 
\end{align}
which measure the judge error and group information not captured by $\phi(X)$ within band $I$ and the representation-level group-approximation error.
We now show how these components bound the bias:
\begin{restatable}[Residual bias decomposition]{theorem}{residual}
\label{thm:residual_alignment}
Let $g\in\calG$ and fix $h\in\calH$ and define
\(
\tau_h(f)^2
=
\bbE\!\left[
h(\phi(X))^2 f(X)(1-f(X))
\right]
\)
Assume that $f$ is multicalibrated over $(\mathcal H,\mathcal I)$, i.e.,
\(
\left|
\bbE\!\left[
(Y-f(X))h'(\phi(X))\indicator{B_I}
\right]
\right|
\le
\alpha\tau_{h'}(f)
\) for all
\(h'\in\mathcal H,\ I\in\mathcal I
\)
Fix $h_1,\ldots,h_m\in\mathcal H$ and a constant $C>0$. For every $I\in\mathcal I$ with $\bbP(B_I)>0$, let
$c_I\in\mathbb R^m$ satisfy $\|c_I\|_1\le C$ and define
\[
e_I
=
\bbE\!\Biggl[
\Biggl(
\bbE\!\Bigl[
\bbE[Y\mid X]-f(X)
\mid
\phi(X),B_I
\Bigr]
-
\sum_{j=1}^m c_{I,j}h_j(\phi(X))
\Biggr)^2
\Biggm| B_I
\Biggr]^{1/2}. \notag
\]
with $\chi_I=e_I + \sqrt{\nicefrac{C\alpha}{\mathbb P(B_I)}}$.
Then for every $g \in \calG$ and $I \in \calI$ such that \(\bbP(g(X)=1, B_I)>0\),
\[
\operatorname{bias}_{g,\mathcal I}(f)
\le
\max_{I\in\mathcal I}
\frac{
\alpha\tau_h(f)
+
\bbP(B_I)
\Biggl[
\lambda_I\varphi_I
+
\psi_I \chi_I\Biggr]}
{
\bbP(g(X)=1,B_I)
} 
\notag
\]
\end{restatable}
\begin{proof}
We provide the proof in~\cref{app:proof_thm1}. 
\end{proof}

The bound separates three sources of residual bias. The first term, $\alpha\tau_{h}(f)$, is the calibration violation that remains along the representation-space group $h$. 
This is controlled directly by multicalibration. Classical multicalibration procedures iteratively identify residual violations over a function class and update the predictor along them, with each non-trivial correction yielding progress in squared prediction error~\citep{hebert2018multicalibration}. Viewing $\calH$ as the multicalibrator's function class over $\phi(X)$, the residual-orthogonality assumption therefore captures the multicalibration machinery.
The second term, $\lambda_I\varphi_I$, couples the portions of the judge's error and group membership that are not captured by the representation. 
Finally, $\chi_I\psi_I$ couples group-approximation error with the systematic error visible in representation space ($e_I)$ plus a multicalibration penalty.

Importantly, the latter two contributions depend on \emph{products}. Hence, accurate recovery of group membership is not necessary for bias control: $\lambda_I\varphi_I$ can be small if either the relevant error structure or the group structure is well captured by the representation, while $\chi_I\psi_I$ can be small if either little systematic error ($e_I$) remains visible in representation space or $h$ accurately captures the relevant group component.
Thus, poor input-space group recovery alone does not imply poor bias mitigation, what matters is how the information preserved in the representation interacts with the judge's error.

\subsection{The Value of Explicit Bias Specification}
\label{subsec:named_groups}

Perhaps surprisingly, eliminating the multicalibration term need not mitigate the bias of $g$. Even under exact residual orthogonality to $h$, the interaction between the unrecovered error and group components may remain non-zero. The following proposition shows that this term is unavoidable.

\begin{restatable}[Necessity and tightness of the unrecoverable interaction]
{proposition}{tightness}
\label{prop:tightness}
Assume that $\calI$ contains a confidence band including
$\nicefrac12$, that $\calH$ contains the constant function
$h\equiv\nicefrac12$, and that
$\operatorname{bias}_{g,\calI}(f)$ is evaluated only over
$I\in\calI$ such that
$\bbP(g(X)=1,f(X)\in I)>0$.
Then, for every $p_g\in(0,1)$ and every $0<\rho\le\frac12\min\{p_g,1-p_g\}$, there exist a distribution, a group $g$ with $\bbP(g(X)=1)=p_g$, a representation $\phi(X)$ and a predictor $f\equiv\nicefrac12$ such that
$f$ is perfectly calibrated and the residual is orthogonal to $h$ on every $I\in\calI$. Moreover, for every $I\in\calI$,
\(
\lambda_I\varphi_I=\rho\), \(
\chi_I\psi_I=0,
\)
while
\[
\operatorname{bias}_{g,\calI}(f)
=
\frac{\lambda_I\varphi_I}{p_g}
=
\frac{\rho}{p_g}\notag
\]
\end{restatable}

\begin{proof}
We provide the proof in~\cref{app:new_proof_prop_tight}. 
\end{proof}

Consequently, the term $\lambda_I\varphi_I$ cannot in general be removed from the bound of~\cref{thm:residual_alignment}. However, when a documented bias such as position, verbosity, or family is known, its group indicator can be supplied directly to the multicalibrator. Augmenting the representation with $g(X)$ makes the group exactly identifiable and eliminates both group-side approximation terms, $\varphi_I$ and $\psi_I$.
\begin{corollary:explicit_bias}[What naming a bias is worth]
\label{cor:explicit_bias}
Let $g\in\calG$ be a bias group and consider
\(
\tilde\phi(X)=[\phi(X),g(X)].
\)
Assume that $\calH$ contains the coordinate projection
$h_g(\tilde\phi(X))=g(X)$. Then, under the setting and assumptions of
\cref{thm:residual_alignment},
\(
\varphi_I=\psi_I=0
\) for all $I \in \calI$
and therefore:
\[
\operatorname{bias}_{g,\calI}(f)
\le
\max_{I\in\calI}
\frac{\alpha\tau_g(f)}{\bbP(g(X)=1,B_I)} \notag
\]
\end{corollary:explicit_bias}
\begin{proof}
We provide the proof in~\cref{app:explicit}. 
\end{proof}

The corollary isolates \emph{what naming a bias is worth}: once $g(X)$ is provided explicitly, both interaction terms in~\cref{thm:residual_alignment} vanish, leaving only the multicalibration term. Thus, explicit bias specification removes the representation-induced approximation error for that group while preserving the broader protection afforded by the learned
representation.

\subsection{Predictive Performance Guarantees for Multicalibrated LLM-Judges}
\label{subsec:perf_guarantees}
Multicalibration should not come at the expense of predictive performance. 
Decision-theoretic analyses of multicalibrated predictors have previously established general guarantees on downstream decision quality~\citep{baldeschi2025multicalibration}. In the binary classification setting considered here, we derive an explicit accuracy-improvement bound that separates the gain induced by changed decisions from the residual slack permitted by approximate multicalibration.
Our analysis begins denoting the two sets of instances whose decision \emph{changes} after multicalibration with
\(
  \calC^\uparrow \coloneqq \{\mbx : f_0(\mbx) < \tfrac12 \le f^\star(\mbx)\}
\)
and
\(
  \calC^\downarrow \coloneqq \{\mbx : f^\star(\mbx) < \tfrac12 \le f_0(\mbx)\}. 
\)
We denote with $\calC := \calC^\uparrow \cup \calC^\downarrow$ their union. This is the subset of instances such that the two judges $f_0$ and $f^\star$ disagree. 

Any result on the net accuracy change between the two judges is restricted to this subset of instances. 
We decompose the predictive performance variation within $\calC$ in two interpretable terms. 
The first term ($\Gamma$) is the predicted-utility advantage of the decision rule induced by $f^\star$ over that induced by $f_0$. It equals the decision margin of instances whose confidences cross the threshold. The second term ($\zeta$) instead accounts for the residual discrepancy between predicted and true utility permitted by approximate multicalibration.
Whenever $\Gamma > \zeta$, the net accuracy change is strictly positive. 
\begin{theorem}[Accuracy improvement under multicalibration]
\label{thm:acc-improve}
Suppose that $f^\star$ is $(\calG^\star,\calI,\alpha)$-multicalibrated,
where $\calI$ contains the intervals $[0,\frac12)$ and
$[\frac12,1]$. Let
\(
g_-(X)=\indicator{f_0(X)<\tfrac12}
\) and 
\(
g_+(X)=\indicator{f_0(X)\ge\tfrac12}
\) denote the two groups induced by the base judge's hard prediction and let $\{g_-,g_+\}\subseteq\calG^\star$.
Let $\tau^\uparrow$ and $\tau^\downarrow$ denote the corresponding
multicalibration scale parameters. Then:
\[
  \mathrm{Acc}(f^\star) - \mathrm{Acc}(f_0) \;\ge\; \underbrace{\bbE\Biggl[\,\Bigl| 2f^\star(X) - 1\Bigr|\;\indicator{X\in \calC}\Biggr]}_{\Gamma} - \underbrace{2\,\alpha\,\Biggl(\tau^\uparrow + \tau^\downarrow \Biggr)}_{\zeta} \notag
\]
\end{theorem}

\begin{proof}
We provide proof in~\cref{app:proof_thm2}.
\end{proof}
The bound reads as a contest between two quantities: the magnitude of the confidence margin for new decisions ($\Gamma$), against how much room multicalibration leaves it to be wrong ($\zeta$). Accuracy strictly improves whenever the first exceeds the second, and, if $\Gamma > 0$, in the perfectly multicalibrated limit where $\alpha \to 0$ \emph{any} decision change is a strict gain.

\section{Experimental Evaluation}
\label{sec:experiments}

In the following, we address three research questions:
\begin{itemize}
    \item \textbf{Q1.} Does \method reduce the biases of the LLM-as-a-judge over existing mitigation strategies without having direct access to the bias groups?
\item \textbf{Q2.} Does \method's bias mitigation further improve when given access to the bias groups?
     \item \textbf{Q3.} Does \method improve over the baselines' predictive performance metrics?
\end{itemize}
The code\footnote{Code fully developed by the authors from KU Leuven and University of Trento.} can be found at~\url{https://github.com/stralu9/CaLLM}.

\subsection{Experimental Settings}
\textbf{Datasets.} 
We consider five pairwise evaluation datasets from the literature. \texttt{RewardBench} (\texttt{RWB}) consists of $3$k prompt-chosen-rejected trios~\citep{DBLP:conf/naacl/LambertPMMLCDKZCSH25}. \texttt{MTBench} (\texttt{MTB}) consists of 709 pairwise comparisons over 80 questions ~\citep{zheng2023judging}. 
Each comparison has LLM-generated answers that are judged pairwise by one or multiple human experts.
\texttt{ArenaExpl} (\texttt{AEX}) is a corpus of $53.8$k pairwise comparisons extracted from conversational data~\citep{DBLP:conf/icml/chatarena}. Here users write the prompt, receive two LLM-generated answers, and act as the judge by expressing a preference. 
Similarly, \texttt{PPE} contains $10.2$k prompts in different languages, with users casting a preference vote for one of two options generated by two LLMs~\citep{DBLP:conf/iclr/Frick25}. 
Finally, \texttt{PKU} consists of $73.9$k pairwise comparisons over safety-oriented prompts. These combine human demonstrations with LLMs~\citep{DBLP:conf/acl/JiPKU}. 
We use the more helpful answer as ground truth. 
More details in \cref{app:data}.

\paragraph{Implementation details.} We rely on \mcgrad~\citep{DBLP:conf/kdd/TaxPLHKOVA26} as the default multicalibrator for \method, because it can automatically search the supplied feature space and thereby uncover undocumented, latent sources of bias. We ablate the choice of multicalibrator and compare with popular calibration techniques in~\cref{app:ablate_multi}.  
To exploit \mcgrad's ability to discover groups automatically, we encode both candidate responses using the open-source \texttt{Muse-Glimmer-30B}~\citep{DBLP:journals/corr/abs-2606-12429}. We provide results with \texttt{Qwen3-0.6B}~\citep{DBLP:journals/corr/abs-2505-09388} in~\cref{app:embedder_ablation}. We denote the resulting embeddings with \(e_1, e_2 \in \mathbb{R}^d\). 
Following previous work, we represent each comparison through the Principal Component Analysis (PCA) of the difference embedding~\citep{DBLP:journals/corr/abs-2606-19057}, denoted as
\( 
E(X) = {\mathrm{PCA}}(e_1-e_2 ).
\) 
We then feed $\phi(X)=[E(X), f_0(X)]$ to \mcgrad.
The former suppresses components shared by the two responses and emphasizes their relative attributes, which are more directly aligned with the structure of pairwise biases. The latter allows to discover miscalibration concealed within confidence bands $\calI$. We provide further details in~\cref{app:impl_details}.

\paragraph{Baselines.} We compare \method against five baselines: $(i)$ the raw verbalized confidence $(f_0)$, which serves as our reference level; $(ii)$
\texttt{BPC}~\citep{wang2024large}, which queries the judge with both response orderings and averages the two confidences; $(iii)$
\texttt{PORTIA}~\citep{li2024split}, which interleaves the two responses in a single prompt such that neither of them occupies a privileged position; $(iv)$ 
\texttt{CalibraEval}~\citep{li2025calibraeval}, which queries the judge under swapped orderings and learns a non-parametric, order-preserving mapping that removes the preference for one position; 
$(v)$ \texttt{LenControl}~\citep{dubois2024length}, which models the judge's preference as a function of response-length difference and then predicts as if both responses had equal length.

\paragraph{Metrics.}
For \textbf{Q1} and \textbf{Q2}, we report the maximum error among the groups induced by the bias-inducing attribute $a$, which can be seen as a per-documented-bias Multicalibration Error ($\text{MCE}_{\%}$). We report $\text{MCE}_{\sigma}$ in~\cref{app:detailed_results}. We compute and report this quantity separately for each documented bias. 
For \textit{verbosity} and \textit{family} biases, we set $\pi$ as the original ordering in which the candidate responses are presented to the judge. This is not suitable for \textit{position} bias, since every comparison necessarily contains a first and a second response. Thus, we define $\pi$ using the alphabetical ordering of the responses source-model names, enabling meaningful position-based groups. 
Following prior work~\citep{DBLP:conf/kdd/TaxPLHKOVA26}, we report the $\text{MCE}_{\%}$ over random partitions of the representation of $X$ in~\cref{app:unspec}. This shows if the multicalibrator can uncover and mitigate errors on unspecified groups.
For \textbf{Q3}, we report accuracy ($\text{ACC}$) together with the log-loss ($\mathcal L_{\text{log}}$) of the judges. For completeness, we report the full results for \textbf{Q1}, \textbf{Q2}, and \textbf{Q3} in~\cref{app:detailed_results}.

\paragraph{Setup.} 
We use \texttt{Muse-Glimmer-30B}, a recent open-weight causal language model that achieves competitive performance among models of comparable size across a broad range of reasoning and agentic benchmarks~\citep{meta_muse_glimmer_2026}, as a judge. We ablate the judge choice with \texttt{Qwen} \texttt{2.5-7B}~\citep{qwen2.5} in~\cref{app:judge_ablation}. Each instance $X=(Q, R_1, R_2)$ is judged once, extracting the judge's verbalized confidence~\citep{tian2023just} $f_0(X)$ that $R_1$ is the preferred answer. This choice allows our framework to extend to the black-box setting, where one lacks access to internal model weights. We then either multicalibrate the base scores $f_0(X)$ or apply a baseline method. 
Concretely, we perform 5-fold cross validation, making sure that different answer pairs for the same prompt always belong to the same fold.

\subsection{Experimental Results}

\begin{figure}[t]
    \centering
\includegraphics[width=\linewidth]{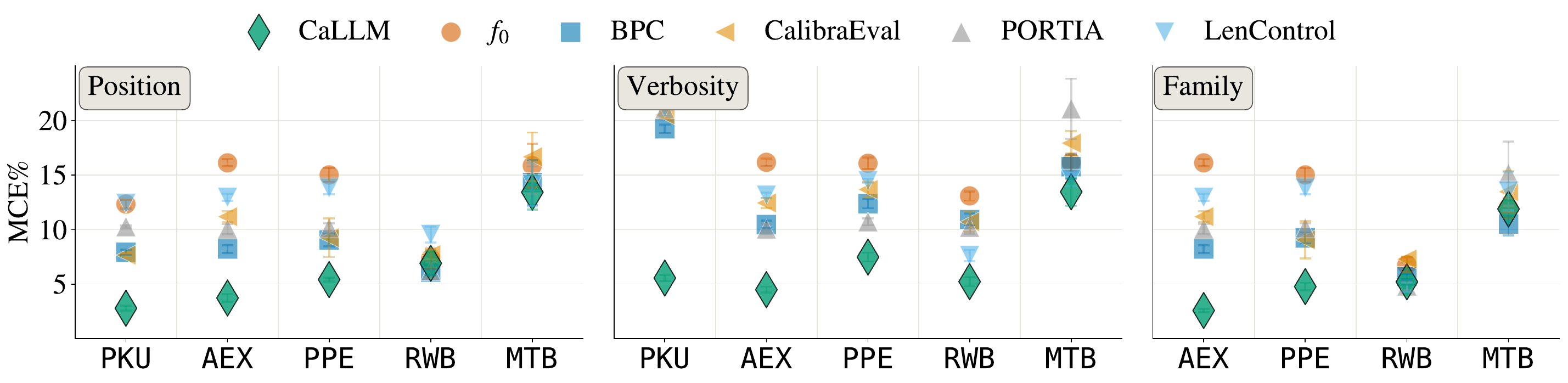}
    \caption{Average per-axis $\text{MCE}_\%$ ($\pm$ s.e.) for each method across datasets. \texttt{PKU} doesn't contain the information regarding the LLMs that generated the responses, so we do not report \textit{Family} bias. }
    \label{fig:Q1}
\end{figure}
\textbf{Q1: \method shows no significant miscalibration in 12 out of 14 evaluations.} We analyze whether \method reduces the \textit{Position}, \textit{Verbosity} and \textit{Family} biases. \cref{fig:Q1} reports the average per-group $MCE_\%$ across datasets 
(see~\cref{tab:APP_q1_MCE_perc_and_sigma} for full results). We additionally validate our choice of \mcgrad as main multicalibration method and compare it with global calibration approaches that also leverage labelled calibration data in~\cref{app:ablate_multi}. \method outperforms global calibration methods on all 14 dataset-groups pairs, and \mcgrad is the best multicalibration technique on 11 of 14 of them. In~\cref{app:recover}, we assess whether \method can recover groups from the available representation.

We observe that \method achieves the lowest error in 11 out of 14 datasets and is the only method that substantially reduces all three biases, with a median reduction with respect to the uncalibrated model $f_0$ of $\approx64\%$, $\approx60\%$, and $\approx46\%$ for position, verbosity, and family, vs. at most $\approx36\%$, $\approx22\%$, and $\approx31\%$ for the other baselines. Interestingly, we observe the largest gains on larger datasets: \eg on \texttt{PKU}, \method achieves $\text{MCE}_\%\approx2.8\pm.2$ for \textit{Position} (second-best \texttt{CalibraEval} with $\text{MCE}_\%\approx7.7\pm.2$) and $\text{MCE}_\%\approx5.6\pm.3$  for \textit{Verbosity} (second-best \texttt{BPC} with $\text{MCE}_\%\approx19.2\pm.4$). Conversely, the three exceptions occur on the two smallest datasets: on \texttt{RWB}, where $f_0$ shows no strong evidence of \textit{Position} bias ($\mathrm{MCE}_\sigma=3.4$), and two of the three position-specific baselines (\texttt{BPC} and \texttt{PORTIA}) outperform \method on position. Notably, explicitly providing CaLLM with the position groups reverses this ranking (see \textbf{Q2}); for \textit{Family} bias,  \method is within one standard error of the best method on both \texttt{RWB} and \texttt{MTB}.
Notably, \method leaves no strong evidence of miscalibration ($\mathrm{MCE}_\sigma<4.71$)
\footnote{$\mathrm{MCE}_\sigma=4.71$ equals a conservative p-value of $10^{-5}$ under the null hypothesis of perfect calibration.} in 12 out of 14 dataset pairs (reference level was 36\% for $f_0$), with the only exceptions occurring for \textit{Verbosity} on \texttt{PKU} and \texttt{AEX}, the two largest datasets, where even small deviations are detectable (see \Cref{tab:APP_q1_MCE_perc_and_sigma}).

\begin{table}[t]
\normalsize
\setlength{\tabcolsep}{4pt}
    \centering
    \caption{Comparison between \method and \methodG{} in terms of $MCE_{\%}$ for documented biases. \texttt{PKU} doesn't contain the information regarding the LLMs that generated the responses, so we do not report \textit{Family} bias.}
    \scalebox{0.7}{ 
\begin{tabular}{l|cc|cc|cc|cc|cc}
\toprule
 & \multicolumn{2}{c|}{PKU} & \multicolumn{2}{c|}{AEX} & \multicolumn{2}{c|}{PPE} & \multicolumn{2}{c|}{RWB} & \multicolumn{2}{c}{MTB} \\
 & \method & \method$+g$ & \method & \method$+g$ & \method & \method$+g$ & \method & \method$+g$ & \method & \method$+g$ \\
\midrule
\textit{Position} & 2.8 $\pm$ 0.2 & \textbf{2.6 $\pm$ 0.1} & 3.7 $\pm$ 0.3 & \textbf{2.5 $\pm$ 0.3} & 5.4 $\pm$ 0.2 & \textbf{3.7 $\pm$ 0.3} & 6.9 $\pm$ 1.1 & \textbf{4.7 $\pm$ 0.5} & 13.4 $\pm$ 1.6 & \textbf{12.5 $\pm$ 1.5} \\
\textit{Verbosity} & 5.6 $\pm$ 0.3 & \textbf{2.8 $\pm$ 0.2} & 4.5 $\pm$ 0.3 & \textbf{2.5 $\pm$ 0.2} & 7.5 $\pm$ 0.4 & \textbf{4.0 $\pm$ 0.2} & 5.2 $\pm$ 0.4 & \textbf{4.7 $\pm$ 0.4} & 13.5 $\pm$ 1.3 & \textbf{12.3 $\pm$ 0.9} \\
\textit{Family} & -- & -- & 2.6 $\pm$ 0.2 & \textbf{2.5 $\pm$ 0.2} & 4.8 $\pm$ 0.3 & \textbf{4.6 $\pm$ 0.4} & 5.2 $\pm$ 0.3 & \textbf{4.6 $\pm$ 0.1} & 11.9 $\pm$ 0.8 & \textbf{11.6 $\pm$ 0.2} \\

\bottomrule
\end{tabular}}
    \label{tab:Q2}
\end{table}
\begin{table}[t]
    \caption{Average ACC and log-loss $\mathcal{L}_{\log}$ ($\pm$ s.e.) across datasets. \method always matches or improves both metrics over all baselines, showing that multicalibration also improves predictive performance.}
\setlength{\tabcolsep}{4pt}
\normalsize
    \centering
    \scalebox{0.7}{
    \begin{tabular}{l|cc|cc|cc|cc|cc}
\toprule
 & \multicolumn{2}{c|}{PKU} & \multicolumn{2}{c|}{AEX} & \multicolumn{2}{c|}{PPE} & \multicolumn{2}{c|}{RWB} & \multicolumn{2}{c}{MTB} \\
 & ACC & $\mathcal{L}_{\log}$ & ACC & $\mathcal{L}_{\log}$ & ACC & $\mathcal{L}_{\log}$ & ACC & $\mathcal{L}_{\log}$ & ACC & $\mathcal{L}_{\log}$ \\
\midrule
\method & \textbf{.69 $\pm$ .01} & \textbf{.57 $\pm$ .01} & \textbf{.68 $\pm$ .01} & \textbf{.60 $\pm$ .01} & \textbf{.67 $\pm$ .01} & \textbf{.61 $\pm$ .01} & \textbf{.93 $\pm$ .01} & \textbf{.21 $\pm$ .01} & \textbf{.81 $\pm$ .01} & \textbf{.41 $\pm$ .02} \\
\midrule
$f_0$ & .63 $\pm$ .01 & .73 $\pm$ .01 & .63 $\pm$ .01 & .76 $\pm$ .01 & .64 $\pm$ .01 & .74 $\pm$ .01 & .86 $\pm$ .01 & .36 $\pm$ .02 & .79 $\pm$ .01 & .49 $\pm$ .02 \\
\texttt{BPC} & .64 $\pm$ .01 & .67 $\pm$ .01 & .65 $\pm$ .01 & .69 $\pm$ .01 & .65 $\pm$ .01 & .68 $\pm$ .01 & .88 $\pm$ .01 & .29 $\pm$ .01 & .80 $\pm$ .01 & .43 $\pm$ .02 \\
\texttt{CalibraEval} & .64 $\pm$ .01 & .65 $\pm$ .01 & .65 $\pm$ .01 & .69 $\pm$ .01 & .65 $\pm$ .01 & .66 $\pm$ .01 & .88 $\pm$ .01 & .30 $\pm$ .01 & .81 $\pm$ .01 & .43 $\pm$ .02 \\
\texttt{PORTIA} & .61 $\pm$ .01 & .72 $\pm$ .01 & .65 $\pm$ .01 & .72 $\pm$ .01 & .66 $\pm$ .01 & .70 $\pm$ .01 & .87 $\pm$ .01 & .31 $\pm$ .02 & .77 $\pm$ .02 & .55 $\pm$ .05 \\
\texttt{LenControl} & .64 $\pm$ .01 & .67 $\pm$ .01 & .65 $\pm$ .01 & .69 $\pm$ .01 & .65 $\pm$ .01 & .68 $\pm$ .01 & .88 $\pm$ .01 & .31 $\pm$ .01 & .80 $\pm$ .01 & .43 $\pm$ .02 \\
\bottomrule
\end{tabular}
}
    \label{tab:q3}
\end{table}
\textbf{Q2: Naming the groups always lowers the per-bias $\mathrm{MCE}_\%$.} We quantify \textit{what naming a bias is worth} by comparing \method with an augmented version (\methodG{}) that adds to $\phi(X)$ the group indicators resulting from the \textit{Position}, \textit{Verbosity} and \textit{Family} biases (see ~\cref{tab:Q2}).

More precisely, we can see that for \textit{Position} bias, \methodG{} further reduces $\text{MCE}_\%$ on all five datasets, with the largest gains on \texttt{AEX} ($\text{MCE}_\%$ passes from $\approx3.7\pm0.3$ to $\approx2.5\pm0.3$, a drop of $\approx33\%$ ) and \texttt{RWB} ($\text{MCE}_\%$ passes from $\approx6.9\pm1.1$ to $\approx4.7\pm0.5$, a drop of $\approx32\%$ ).

For \textit{Verbosity} bias, \methodG{} exhibits even stronger gains, with an average reduction of $\approx 32\%$ across the five datasets. In particular, we observe the largest improvements on \texttt{PKU} ($\approx49.4\%$), \texttt{PPE} ($\approx46.7\%$) and \texttt{AEX} ($\approx 43.2\%$), the three largest datasets of our collection.
For \textit{Family} bias, we observe smaller improvements, likely due to the fact that \method already performed quite well across all datasets. Still, we can observe some gains on smaller datasets: \eg on \texttt{RWB} $\text{MCE}_\%$ passes from $\approx 5.2\pm0.3$ to $\approx 4.6\pm 0.1$. Overall, \methodG{} completely removes evidence of miscalibration on all datasets by reducing the $\mathrm{MCE}_\sigma$ to $<4.71$ (compared to 12/14 for \method, see \cref{tab:APP_q2_MCE_perc_and_sigma}).

\textbf{Q3: A \method judge is also a better judge.} \Cref{tab:q3} displays the results in terms of predictive performance, reporting both $\text{ACC}$ and $\mathcal L_{\text{log}}$ for all the methods. We can see that \method achieves the highest accuracy on four out of five datasets (matching the best baseline on the fifth) and the lowest log-loss on all five datasets.

Compared with $f_0$, $\text{ACC}$ increases by $\approx4.6$ percentage points on average, while $\mathcal L_{\text{log}}$ decreases by $\approx24\%$. Consistently with previous results, largest gains are obtained on largest data: on \texttt{PKU}, \method passes from $\approx.63\pm.01$ to $\approx.69\pm.01$, a $\approx9.8\%$ increase; on \texttt{AEX}, \method reaches $\approx .68\pm.01$ starting from a base model of $\approx.63$, an improvement of $\approx7\%$. We observe similar patterns also for $\mathcal L_{\text{log}}$, with \method consistently outperforming competitors. The largest gain is on \texttt{RWB}, where $L_{\text{log}}$ drops by $42\%$, passing from the base level of $\approx.36 \pm .02$ to $\approx.21\pm.01$.
Overall, these results suggest that \method effectively enforces predictive performance gains.

\section{Related Work}
\label{sec:related}

\textbf{LLM judges and bias mitigation.} LLM judges provide a scalable alternative to human evaluation through pointwise, pairwise, and reference-guided protocols~\citep{DBLP:conf/emnlp/LiuIXWXZ23,DBLP:conf/eacl/LiusieMG24}. Their judgments, however, are affected by position, verbosity, model family, and other incidental attributes~\citep{wang2024large,panickssery2024llm,koo2024benchmarking,chen2024humans,ye2025justice}. Existing mitigation strategies modify inference through response swapping, interleaving, panels, or decomposed criteria~\citep{wang2024large,li2024split,verga2024replacing,chan2024chateval,yang2026quantifying}; fine-tune evaluators~\citep{kim2024prometheus,vu2024foundational,zhu2025judgelm}; or post-process confidence scores~\citep{dubois2024length,li2025calibraeval}. These methods target predefined biases and may require additional inference or annotated data~\citep{li2025generation}. Conversely, \method is post-hoc and controls any bias whose error pattern is recoverable from the supplied representation.

\textbf{Calibration.} 
Classical parametric and non-parametric calibration methods learn a global correction~\citep{DBLP:conf/aistats/VaicenaviciusWA19,platt1999probabilistic,DBLP:conf/icml/GuoPSW17,kull2017beta,DBLP:conf/kdd/ZadroznyE02,DBLP:conf/icml/ZadroznyE01}, with extensions to multiclass scores~\citep{DBLP:conf/nips/KullPKFSF19,DBLP:journals/corr/abs-2511-03685}. Local calibration refines this incorporating a notion of distance between inputs~\citep{DBLP:conf/uai/LuoBBZWXSESP22,barbera26,DBLP:conf/iclr/LinT023,DBLP:journals/tmlr/ValkK23,barbera2026divide}. Multicalibration instead requires calibration simultaneously across groups identified by a hypothesis class~\citep{hebert2018multicalibration,jung2021moment,baldeschi2025multicalibration,jin2025discretization,DBLP:conf/kdd/TaxPLHKOVA26}. This group-wise formulation matches how LLM-judge biases arise and allows \method to provide guarantees beyond a predefined list of biases~\citep{globus2023multicalibrated,noarov2023statistical}.

\section{Conclusion}

We tackled bias in LLM-as-a-judge by casting it as a multicalibration problem. We introduced \method, a post-hoc framework that uses the representation of each comparison to identify and correct groups with systematic residual error. We showed that \method bounds every bias group whose residual structure is recoverable from the representation and that explicitly providing a documented group further improves the corresponding guarantees. We showed that multicalibration improves accuracy when changed-decision margins exceed residual slack, which is always true under perfect multicalibration. Across five benchmark data, \method reduces bias on documented biases while improving accuracy and log loss over the raw judge and bias-specific baselines.

\textbf{Limitations and future work.}
Our guarantees require the pattern of target biases to be recoverable from functions of the chosen representation. Biases whose residual structure is not captured may therefore remain uncontrolled. Moreover, our guarantees are population-level and become weaker for rare groups, where calibration errors are also harder to estimate from finite data. 
Our notion of bias is distributional: \method corrects systematic group-wise discrepancies but does not identify whether the corresponding attributes causally influence the judge, suggesting a complementary connection with counterfactual notions of bias. Finally, we focus on binary pairwise evaluation. We leave extensions to pointwise scoring, multi-way comparisons, and ranking-based LLM evaluation as future work.

\section*{Acknowledments}
Luca Stradiotti and Jesse Davis were supported by the Flemish Government through the “Onderzoeksprogramma Artificiële Intelligentie (AI) Vlaanderen” programme. Jesse Davis was also supported by the KU Leuven Research Fund (C14/24/091). Lorenzo Cascioli was supported by the Research Foundation Flanders (FWO, 11I8125N). 
Cesare Barbera, Andrea Pugnana and Andrea Passerini acknowledge the CINECA award under the ISCRA initiative for the availability of high-performance computing resources and support. This work was funded by the European Union. The views and opinions expressed are however those of the author(s) only and do not necessarily reflect those of the European Union, the European Health and Digital Executive Agency (HaDEA) or the European Research Executive Agency. Neither the European Union nor the granting authority can be held responsible for them. Grant Agreement no. 101120763 - TANGO.

\section*{AI use statement}
We used generative AI tools to review and improve the writing and presentation of the mathematical proofs, provide feedback on the experimental methodology, and assist with the implementation of the experiment code. We also used these tools to help prepare Figure~1. We reviewed all AI-assisted text, proofs, code, and figures, verified the resulting arguments and experimental outputs, and take full responsibility for the final content of this work.

\section*{Ethics statement}
This work aims to make LLM-based evaluation less sensitive to systematic biases. However, lower group-wise calibration error should not be interpreted as a guarantee of fairness: \method may not protect groups that are absent from the calibration data or whose error structure cannot be recovered from the chosen representation, and it inherits limitations in the human preference labels used for calibration. The method should therefore complement, rather than replace, human oversight in high-stakes settings.

\section*{Reproducibility statement}
The main text describes the datasets, judge, confidence-elicitation procedure, baselines, metrics, representation, and five-fold evaluation protocol used in our experiments. Complete proofs are provided in the appendix, together with additional experimental results. The implementation and experiment scripts used to produce our results are attached with the submission.

\bibliography{bibliography}
\bibliographystyle{unsrtnat}

\appendix

\crefname{appendix}{Appendix}{Appendices}
\Crefname{appendix}{Appendix}{Appendices}

\section{Proofs}
\label[appendix]{app:proofs}

\subsection{Proof of Theorem 1}
\label[appendix]{app:proof_thm1}

\residual*

\begin{proof}
Fix $I\in\calI$ such that
\(
p_{g,I}:=\bbP(g(X)=1,f(X)\in I)>0,
\)
and let
\(
B_I:=\{f(X)\in I\}.
\)
For compactness, define
\[
\etaxp:=\etax,
\]
and, within the band $B_I$,
\[
\mu_I(\phi(X))
:=
\bbE[\etaxp-f(X)\mid \phi(X),B_I],
\qquad
\gamma_I(\phi(X))
:=
\bbE[g(X)\mid \phi(X),B_I].
\]

Since $g(X)$, $f(X)$, and $\indicator{B_I}$ are $X$-measurable,
the tower property gives
\[
\bbE\!\left[
(Y-f(X))g(X)\indicator{B_I}
\right]
=
\bbE\!\left[
(\etaxp-f(X))g(X)\indicator{B_I}
\right].
\]
Similarly, since $h(\phi(X))$ is $X$-measurable,
\[
\bbE\!\left[
(Y-f(X))h(\phi(X))\indicator{B_I}
\right]
=
\bbE\!\left[
(\etaxp-f(X))h(\phi(X))\indicator{B_I}
\right].
\]
Hence,
\begin{align*}
& \bbE\!\left[
(Y-f(X))
\bigl(g(X)-h(\phi(X))\bigr)
\indicator{B_I}
\right] =
\bbE\!\left[
(\etaxp-f(X))
\bigl(g(X)-h(\phi(X))\bigr)
\indicator{B_I}
\right].
\end{align*}

Conditioning on $B_I$ and adding and subtracting
$\mu_I(\phi(X))$ and $\gamma_I(\phi(X))$, we obtain
\begin{align*}
&
\bbE\!\left[
(\etaxp-f(X))
\bigl(g(X)-h(\phi(X))\bigr)
\mid B_I
\right]\\
&=
\bbE\!\left[
\bigl(\etaxp-f(X)-\mu_I(\phi(X))\bigr)
\bigl(g(X)-\gamma_I(\phi(X))\bigr)
\mid B_I
\right]
\\
&+
\bbE\!\left[
\mu_I(\phi(X))
\bigl(\gamma_I(\phi(X))-h(\phi(X))\bigr)
\mid B_I
\right].
\end{align*}
Indeed, the remaining cross terms vanish by the tower property. First,
\begin{align*}
&
\bbE\!\left[
\bigl(\etaxp-f(X)-\mu_I(\phi(X))\bigr)
\bigl(\gamma_I(\phi(X))-h(\phi(X))\bigr)
\mid B_I
\right]
\\
&=
\bbE\!\left[
\bigl(\gamma_I(\phi(X))-h(\phi(X))\bigr)
\bbE\!\left[
\etaxp-f(X)-\mu_I(\phi(X))
\mid \phi(X),B_I
\right]
\Bigm| B_I
\right]
=0,
\end{align*}
because
\[
\mu_I(\phi(X))
=
\bbE[\etaxp-f(X)\mid\phi(X),B_I].
\]
Likewise,
\begin{align*}
&
\bbE\!\left[
\mu_I(\phi(X))
\bigl(g(X)-\gamma_I(\phi(X))\bigr)
\mid B_I
\right]
\\
&=
\bbE\!\left[
\mu_I(\phi(X))
\bbE\!\left[
g(X)-\gamma_I(\phi(X))
\mid \phi(X),B_I
\right]
\Bigm| B_I
\right]
=0.
\end{align*}
because
\[
\gamma_I(\phi(X))
=
\bbE[g(X)\mid\phi(X),B_I].
\]
Multiplying by $\bbP(B_I)$ therefore yields
\begin{align*}
&
\bbE\!\left[
(Y-f(X))
\bigl(g(X)-h(\phi(X))\bigr)
\indicator{B_I}
\right]
\\
&=
\bbP(B_I)
\bbE\!\left[
\bigl(\etaxp-f(X)-\mu_I(\phi(X))\bigr)
\bigl(g(X)-\gamma_I(\phi(X))\bigr)
\mid B_I
\right]
\\
&\quad+
\bbP(B_I)
\bbE\!\left[
\mu_I(\phi(X))
\bigl(\gamma_I(\phi(X))-h(\phi(X))\bigr)
\mid B_I
\right].
\end{align*}
Apply the triangle inequality:
\begin{align*}
&
\Biggl|
\bbE\!\Bigl[
(Y-f(X))
\bigl(g(X)-h(\phi(X))\bigr)
\indicator{B_I}
\Bigr]
\Biggr|
\\
&\le
\bbP(B_I)
\Biggl|
\bbE\!\Bigl[
\bigl(\etaxp-f(X)-\mu_I(\phi(X))\bigr)
\bigl(g(X)-\gamma_I(\phi(X))\bigr)
\mid B_I
\Bigr]
\Biggr|
\\
&\quad+
\bbP(B_I)
\Biggl|
\bbE\!\Bigl[
\mu_I(\phi(X))
\bigl(\gamma_I(\phi(X))-h(\phi(X))\bigr)
\mid B_I
\Bigr]
\Biggr|.
\end{align*}
For the first term Cauchy-Schwarz gives:
\begin{align*}
&
\Biggl|
\bbE\!\Bigl[
\bigl(\etaxp-f(X)-\mu_I(\phi(X))\bigr)
\bigl(g(X)-\gamma_I(\phi(X))\bigr)
\mid B_I
\Bigr]
\Biggr|
\\
&\le
\underbrace{\sqrt{
\bbE\!\left[
\bigl(\etaxp-f(X)-\mu_I(\phi(X))\bigr)^2
\mid B_I
\right]
}}_{\lambda_I}\cdot
\underbrace{\sqrt{
\bbE\!\left[
\bigl(g(X)-\gamma_I(\phi(X))\bigr)^2
\mid B_I
\right]
}}_{\varphi_I}
\end{align*}
For the second term, Cauchy--Schwarz gives:
\begin{align*}
&
\Biggl|
\bbE\!\Bigl[
\mu_I(\phi(X))
\bigl(\gamma_I(\phi(X))-h(\phi(X))\bigr)
\mid B_I
\Bigr]
\Biggr|
\\
&\le
\sqrt{
\bbE\!\left[
\mu_I(\phi(X))^2
\mid B_I
\right]
}
\cdot
\underbrace{\sqrt{
\bbE\!\left[
\bigl(\gamma_I(\phi(X))-h(\phi(X))\bigr)^2
\mid B_I
\right]
}}_{\psi_I}
\\
\end{align*}
The previous two bounds therefore imply
\begin{align}
&
\Biggl|
\bbE\!\left[
(Y-f(X))
\bigl(g(X)-h(\phi(X))\bigr)
\indicator{B_I}
\right]
\Biggr|
\nonumber\\
&\qquad\le
\bbP(B_I)
\left(
\lambda_I\varphi_I
+
\sqrt{\bbE\!\left[
\mu_I(\phi(X))^2
\mid B_I
\right]}\psi_I
\right)
\label{eq:residual_decomposition_intermediate}
\end{align}
It remains to control $
\sqrt{\bbE\!\left[
\mu_I(\phi(X))^2
\mid B_I
\right])}$ using multicalibration.
For the coefficients $c_I$ appearing in the statement, define
\[
k_I(\phi(X))
:=
\sum_{j=1}^m c_{I,j}h_j(\phi(X)),
\]
By the definition of $\mu_I$ and the tower property,
\begin{align}
\bbP(B_I)\bbE\!\left[
\mu_I(\phi(X))^2
\mid B_I
\right]
&=
\bbE\!\left[
\mu_I(\phi(X))^2\indicator{B_I}
\right]
\nonumber\\
&=
\bbE\!\left[
(\etaxp-f(X))
\mu_I(\phi(X))
\indicator{B_I}
\right]
\nonumber\\
&=
\bbE\!\left[
(Y-f(X))
\mu_I(\phi(X))
\indicator{B_I}
\right].
\label{eq:chi_residual_moment}
\end{align}
Adding and subtracting $k_I(\phi(X))$ gives
\begin{align}
\bbP(B_I)\bbE\!\left[
\mu_I(\phi(X))^2
\mid B_I
\right]
&=
\bbE\!\left[
(Y-f(X))
k_I(\phi(X))
\indicator{B_I}
\right]
\nonumber\\
&\quad+
\bbE\!\left[
(Y-f(X))
\bigl(
\mu_I(\phi(X))-k_I(\phi(X))
\bigr)
\indicator{B_I}
\right].
\label{eq:chi_split}
\end{align}
For the first term, linearity and multicalibration yield
\begin{align}
&
\Bigl|
\bbE\!\left[
(Y-f(X))
k_I(\phi(X))
\indicator{B_I}
\right]
\Bigr|
\nonumber\\
&\qquad\le
\sum_{j=1}^m
\Bigl|c_{I,j}\Bigr|
\Bigl|
\bbE\!\left[
(Y-f(X))
h_j(\phi(X))
\indicator{B_I}
\right]
\Bigr|
\nonumber\\
&\qquad\le
\alpha
\sum_{j=1}^m
\Bigl|c_{I,j}\Bigr|\tau_{h_j}(f)
\label{eq:chi_multicalibration}
\end{align}
Since $h_j:\mathbb R^d\rightarrow[0,1]$,
\[
\tau_{h_j}(f)^2
=
\bbE\!\left[
h_j(\phi(X))^2f(X)(1-f(X))
\right]
\le
\frac14
\]
and hence
\[
\sum_{j=1}^m
\Bigl|c_{I,j}\Bigr|\tau_{h_j}(f)
\le
\frac12\|c_I\|_1
\le
\frac C2
\le C
\]
Therefore,
\begin{equation}
\Biggl|
\bbE\!\left[
(Y-f(X))
k_I(\phi(X))
\indicator{B_I}
\right]
\Biggr|
\le
\alpha C
\label{eq:chi_multicalibration_C}
\end{equation}
For the second term, applying the same tower-property identity used in
\cref{eq:chi_residual_moment} gives
\begin{align}
&
\bbE\!\left[
(Y-f(X))
\bigl(
\mu_I(\phi(X))-k_I(\phi(X))
\bigr)
\indicator{B_I}
\right]
\nonumber\\
&=
\bbP(B_I)
\bbE\!\left[
\mu_I(\phi(X))
\bigl(
\mu_I(\phi(X))-k_I(\phi(X))
\bigr)
\mid B_I
\right].
\end{align}
Thus, by Cauchy--Schwarz and the definition of $e_I$,
\begin{align}
&
\Biggl|
\bbE\!\left[
(Y-f(X))
\bigl(
\mu_I(\phi(X))-k_I(\phi(X))
\bigr)
\indicator{B_I}
\right]
\Biggr|
\nonumber\\
&\qquad\le
\bbP(B_I)
\sqrt{
\bbE[\mu_I(\phi(X))^2\mid B_I]
}
\sqrt{
\bbE[
(\mu_I(\phi(X))-k_I(\phi(X)))^2
\mid B_I]
}
\nonumber\\
&\qquad=
\bbP(B_I)\sqrt{\bbE\!\left[
\mu_I(\phi(X))^2 \mid B_I \right])} e_I.
\label{eq:chi_approximation}
\end{align}
Combining
\cref{eq:chi_split,eq:chi_multicalibration_C,eq:chi_approximation},
we obtain
\[
\bbP(B_I)\bbE\!\left[
\mu_I(\phi(X))^2 \mid B_I \right]
\le
\alpha C
+
\bbP(B_I)\sqrt{\bbE\!\left[
\mu_I(\phi(X))^2 \mid B_I \right]} e_I.
\]
Dividing by $\bbP(B_I)>0$,
\[
\bbE\!\left[
\mu_I(\phi(X))^2
\mid B_I
\right]
\le
\sqrt{\bbE\!\left[
\mu_I(\phi(X))^2
\mid B_I
\right]} e_I
+
\frac{\alpha C}{\bbP(B_I)}.
\]
Since $\sqrt{\bbE\!\left[
\mu_I(\phi(X))^2
\mid B_I
\right]}\ge0$ and $e_I\ge0$, solving the corresponding quadratic
inequality gives
\begin{align*}
\sqrt{\bbE\!\left[
\mu_I(\phi(X))^2
\mid B_I
\right]}
&\le
\frac12
\left(
e_I+
\sqrt{
e_I^2+
\frac{4\alpha C}{\bbP(B_I)}
}
\right)\\
&\le
e_I
+
\sqrt{
\frac{\alpha C}{\bbP(B_I)}
},
\end{align*}
where the last inequality follows from
$\sqrt{a^2+b^2}\le a+b$ for $a,b\ge0$.
Therefore,
\begin{equation}
\sqrt{\bbE\!\left[
\mu_I(\phi(X))^2
\mid B_I
\right]}
\le \underbrace{
e_I
+
\sqrt{
\frac{\alpha C}{\bbP(B_I)}
}}_{\chi}
\label{eq:chi_final}
\end{equation}
Substituting \cref{eq:chi_final} into
\cref{eq:residual_decomposition_intermediate} gives
\begin{align}
&
\Biggl|
\bbE\!\left[
(Y-f(X))
\bigl(g(X)-h(\phi(X))\bigr)
\indicator{B_I}
\right]
\Biggr|
\nonumber \le
\bbP(B_I)
\Biggl[
\lambda_I\varphi_I
+
\psi_I
\left(
e_I
+
\sqrt{
\frac{\alpha C}{\bbP(B_I)}
}
\right)
\Biggr].
\label{eq:group_proxy_bound}
\end{align}
Finally,
\begin{align*}
&
\Biggl|
\bbE\!\left[
(Y-f(X))g(X)\indicator{B_I}
\right]
\Biggr|
\\
&\le
\Biggl|
\bbE\!\left[
(Y-f(X))h(\phi(X))\indicator{B_I}
\right]
\Biggr|
\\
&\quad+
\Biggl|
\bbE\!\left[
(Y-f(X))
\bigl(g(X)-h(\phi(X))\bigr)
\indicator{B_I}
\right]
\Biggr|
\\
&\le
\alpha\tau_h(f)
+
\bbP(B_I)
\Biggl[
\lambda_I\varphi_I
+
\psi_I
\left(
e_I
+
\sqrt{
\frac{\alpha C}{\bbP(B_I)}
}
\right)
\Biggr]
\end{align*}
Since $g$ is an indicator,
\[
\Biggl|
\bbE[
Y-f(X)
\mid
g(X)=1,B_I
]
\Biggr|
=
\frac{
\Biggl|
\bbE[
(Y-f(X))g(X)\indicator{B_I}
]
\Biggr|
}{
\bbP(g(X)=1,B_I)
}.
\]
Hence,
\[
\Biggl|
\bbE[
Y-f(X)
\mid
g(X)=1,B_I
]
\Biggr|
\le
\frac{
\alpha\tau_h(f)
+
\bbP(B_I)
\Biggl[
\lambda_I\varphi_I
+
\psi_I
\left(
e_I
+
\sqrt{
\frac{\alpha C}{\bbP(B_I)}
}
\right)
\Biggr]
}{
\bbP(g(X)=1,B_I)
}.
\]
Taking the maximum over all $I\in\mathcal I$ such that
$\bbP(g(X)=1,B_I)>0$ proves the claim.
\end{proof}

\subsection{Proof of Proposition 1}
\label[appendix]{app:new_proof_prop_tight}

\tightness*

\begin{proof}
Let
\[
X\sim\operatorname{Bernoulli}(p_g),
\qquad
g(X)=X,
\qquad
\phi(X)=k, \quad \text{with }k\text{ constant.}
\]
Set
\[
c=\frac{\rho}{p_g(1-p_g)}
\]
and define
\[
Y\mid X
\sim
\operatorname{Bernoulli}
\left(
\frac12+c(X-p_g)
\right),
\qquad
f(X)=\frac12,
\qquad
h(\phi(X))=\frac12
\]
The condition
\[
\rho\le\frac12\min\{p_g,1-p_g\}
\]
ensures that the conditional success probability lies in \([0,1]\).
Indeed, it equals
\[
\frac12+\frac{\rho}{p_g}
\quad\text{when }X=1,
\qquad
\frac12-\frac{\rho}{1-p_g}
\quad\text{when }X=0
\]

Since \(\bbE[X]=p_g\),
\[
\bbE[Y]
=
\frac12+c\bbE[X-p_g]
=
\frac12
\]
Thus, \(f\equiv\nicefrac12\) is perfectly calibrated and $\alpha=0$.

Now fix \(I\in\calI\). If
\(\bbP(f(X)\in I)=0\), residual orthogonality holds trivially.
Otherwise, \(I\) contains \(\nicefrac12\), so
\(\indicator{f(X)\in I}=1\) almost surely. Therefore,
\begin{align*}
\bbE\!\left[
(Y-f(X))h(\phi(X))\indicator{f(X)\in I}
\right]
&=
\frac12\bbE\!\left[Y-\frac12\right]=0
\end{align*}

For every such positive-probability \(I\),
\begin{align*}
\bbE[Y-f(X)\mid g(X)=1,f(X)\in I]
&=
\bbE[Y\mid X=1]-\frac12=
c(1-p_g)=
\frac{\rho}{p_g}
\end{align*}
Hence,
\[
\operatorname{bias}_{g,\calI}(f)
=
\frac{\rho}{p_g}
\]

It remains to evaluate the terms in
\cref{thm:residual_alignment}. We have
\[
\etax-f(X)
=
c(X-p_g)
\]
For any \(I\) with \(\bbP(f(X)\in I)>0\), the event
\(B_I=\{f(X)\in I\}\) occurs almost surely. Since \(\phi(X)\) is
constant,
\[
\bbE[\etax-f(X)\mid\phi(X),B_I]=0,
\qquad
\bbE[g(X)\mid\phi(X),B_I]=p_g
\]
Consequently,
\begin{align*}
\lambda_I^2
&=
c^2\bbE[(X-p_g)^2]
=
c^2p_g(1-p_g),
\\
\varphi_I^2
&=
\bbE[(X-p_g)^2]
=
p_g(1-p_g),
\\
e_I 
&=0
\end{align*}
It follows that
\[
\lambda_I\varphi_I
=
c\,p_g(1-p_g)
=
\rho,
\qquad
\chi_I\psi_I=0
\]

Finally,
\[
\bbP(B_I)=1,
\qquad
\bbP(g(X)=1,f(X)\in I)=p_g
\]
and the multicalibration term is zero by residual orthogonality
Therefore, \cref{thm:residual_alignment} gives
\[
\operatorname{bias}_{g,\calI}(f)
\le
\frac{\lambda_I\varphi_I}{p_g}
=
\frac{\rho}{p_g}
\]
Since the actual bias equals \(\rho/p_g\), the bound is attained with
equality. This concludes the proof.
\end{proof}

\subsection{Proof of Corollary 1}
\label[appendix]{app:explicit}

\begin{corollary:explicit_bias}[What naming a bias is worth]
Let $g\in\calG$ be a bias group and consider
\(
\tilde\phi(X)=[\phi(X),g(X)].
\)
Assume that $\calH$ contains the coordinate projection
$h_g(\tilde\phi(X))=g(X)$. Then, under the setting and assumptions of
\cref{thm:residual_alignment},
\(
\varphi_I=\psi_I=0
\) for all $I \in \calI$
and therefore:
\[
\operatorname{bias}_{g,\calI}(f)
\le
\max_{I\in\calI}
\frac{\alpha\tau_g(f)}{\bbP(g(X)=1,B_I)} \notag
\]
\end{corollary:explicit_bias}

\begin{proof}
Since $g(X)$ is a coordinate of $\tilde\phi(X)$,
\[
\bbE[g(X)\mid \tilde\phi(X),B_I]=g(X),
\]
so $\varphi_I=0$. Moreover, choosing the corresponding coordinate
projection $h_g$ gives
\[
h_g(\tilde\phi(X))=g(X),
\]
hence $\psi_I=0$. 
Since $g$ is a binary map:
\[
h_g(\tilde\phi(X))^2 = g(X)^2 = g(X),
\]
and therefore:
\[
\tau_{h_g}(f)^2 = \bbE\!\left[
h_g(\tilde\phi(X))^2
f(X)\bigl(1-f(X)\bigr)
\right] = \bbE\!\left[
g(X)
f(X)\bigl(1-f(X)\bigr)
\right] = \tau_g(f)^2
\]
Therefore both interaction terms in
\cref{thm:residual_alignment} vanish, yielding
\[
\operatorname{bias}_{g,\calI}(f)
\le
\max_{I\in\calI}
\frac{\alpha\tau_g(f)}{\mathbb P(g(X)=1, f(X)\in I)}
\]
\end{proof}

\subsection{Proof of Theorem 2}
\label[appendix]{app:proof_thm2}
\begin{theorem:acc_improv}[Accuracy improvement under multicalibration]
Suppose that $f^\star$ is $(\calG^\star,\calI,\alpha)$-multicalibrated,
where $\calI$ contains the intervals $[0,\frac12)$ and
$[\frac12,1]$. Let
\(
g_-(X)=\indicator{f_0(X)<\tfrac12}
\) and 
\(
g_+(X)=\indicator{f_0(X)\ge\tfrac12}
\) denote the two groups induced by the base judge's hard prediction and let $\{g_-,g_+\}\subseteq\calG^\star$.
Let $\tau^\uparrow$ and $\tau^\downarrow$ denote the corresponding
multicalibration scale parameters. Then:
\[
  \mathrm{Acc}(f^\star) - \mathrm{Acc}(f_0) \;\ge\; \underbrace{\bbE\Biggl[\,\Bigl| 2f^\star(X) - 1\Bigr|\;\indicator{X\in \calC}\Biggr]}_{\Gamma} - \underbrace{2\,\alpha\,\Biggl(\tau^\uparrow + \tau^\downarrow \Biggr)}_{\zeta} \notag
\]
\end{theorem:acc_improv}

\begin{proof}
The proof of~\cref{thm:acc-improve} relies on two Lemmas illustrated in what follows. First, we characterize the net predictive performance variation between the uncalibrated and the calibrated model.
\begin{lemma}
\label{lemma:lemma1}
The variation in predictive performance between the calibrated judge and the base judge, which depends only on the number of instances inside the set $\mathcal C$, is:
\begin{align}    
\displaystyle \mathrm{Acc}(f^\star) - \mathrm{Acc}(f_0) & = \pup(2\qup-1)\notag\\
&+\pdn(1-2\qdn)
\end{align}
\end{lemma}
\begin{proof}
Given that the correction is bounded, only instances within the set $\calC$ can actually change their prediction. Specifically, a multicalibrated judge $f^\star$ gains accuracy for an instance in $\calC^\uparrow$ if $Y = 1$ and loses it if $Y = 0$, while the opposite applies for $\calC^\downarrow$: 
\begin{equation*}
\begin{split}
\mathrm{Acc}(f^\star) - \mathrm{Acc}(f_0)
&= \pup \left(\qup - (1-\qup)\right) \\
 & \qquad + \pdn \left((1-\qdn) - \qdn\right) \\
&= \pup(2\qup - 1)  \\
& \qquad + \pdn(1 - 2\qdn)
\end{split}
\end{equation*}
\end{proof}
Notably, the flip rates \(\qup\) and \(\qdn\) are controlled by the multicalibration property of \(f^\star\) which promotes alignment with the true label. We characterize this in the following Lemma which establishes two bounds on the population error of the predictions flips. 
\begin{lemma}
\label{lemma:lemma2}
Suppose that $f^\star$ is $(\calG^\star,\calI,\alpha)$-multicalibrated,
where $\calI$ contains the intervals $[0,\frac12)$ and
$[\frac12,1]$. Let
\(
g_-(X)=\indicator{f_0(X)<\tfrac12}
\) and 
\(
g_+(X)=\indicator{f_0(X)\ge\tfrac12}
\) denote the two groups induced by the base judge's hard prediction and let $\{g_-,g_+\}\subseteq\calG^\star$.
Let $\tau^\uparrow$ and $\tau^\downarrow$ denote the corresponding
multicalibration scale parameters. Then:
\[
\qup
\ge
\tfrac12-\frac{\alpha\tau^\uparrow}{\pup},
\qquad
\qdn
\le
\tfrac12+\frac{\alpha\tau^\downarrow}{\pdn}
\]
\end{lemma}

\begin{proof}
We prove the lower bound on $\qup$ in full. The upper bound on $\qdn$ is symmetric.

By definition $\mathcal{C}^\uparrow = \{ \mathbf x : f_0(\mathbf x) < \frac12 \le f^\star(\mathbf x)\}$. 
Let \(I_{+} := \indicator{f^\star(X) \ge \tfrac12}\). Since \(g_- := \indicator{f_0(X) < \tfrac12}\): 
\[
  \indicator{X \in \calC^\uparrow}
  \;=\; \indicator{f_0(X) < \tfrac12}\;\indicator{f^\star(X) \ge \tfrac12}
  \;=\; g_-\,I_{+}.
\]
By assumption $g_- \in \calG^\star$ and $I_+ \in \calI$, so $\indicator{X \in \calC^\uparrow}$ is an admissible group--level-set combination.

Because $f^\star$ is $(\calG^\star, \alpha)$-multicalibrated with respect to $\calI$, the joint residual over $\calC^{\uparrow}$ is bounded by the scale $\tau^\uparrow$:
\[
  \Biggl|\,\bbE\Bigl[(Y - f^\star(X))\,\indicator{X \in \calC^\uparrow}\Bigr]\,\Biggr|
  \;=\;
  \Biggl|\,\bbE\Bigl[(Y - f^\star(X))\,g_-\,I_+\Bigr]\,\Biggr|
  \;\le\; \alpha\,\tau^\uparrow.
\]

Considering the lower side of the absolute value, we obtain
\[
  \bbE[Y\,\indicator{X \in \calC^\uparrow}] - \bbE[f^\star(X)\,\indicator{X \in \calC^\uparrow}]
  \;=\; \bbE\bigl[(Y - f^\star(X))\,\indicator{X \in \calC^\uparrow}\bigr]
  \;\ge\; -\,\alpha\,\tau^\uparrow,
\]

Since $f^\star(\mathbf x) \ge \tfrac12$ over the set $\calC^{\uparrow}$, hence $f^\star(X)\,\indicator{X \in \calC^\uparrow} \ge \tfrac12\,\indicator{X \in \calC^\uparrow}$. Thus,
\begin{equation}
    \bbE[Y\,\indicator{X \in \calC^\uparrow}] \ge \bbE[f^\star(X)\,\indicator{X \in \calC^\uparrow}] - \alpha\,\tau^\uparrow \;\ge\; \tfrac12\,\pup - \alpha\,\tau^\uparrow.
\label{eq:diff}
\end{equation}

Assuming $\pup > 0$ and 
\[
\qup = \frac{\bbP(Y=1,X \in \calC^{\uparrow})}{\pup} = \frac{\bbE[Y\,\indicator{X \in \calC^\uparrow}]}{\pup}
\] 
Then, dividing the previous \cref{eq:diff} by $\pup$ yields
\[
  \qup \;\ge\; \tfrac12 - \frac{\alpha\,\tau^\uparrow}{\pup},
\]
This concludes the proof for the lower bound of $\mathbb{P}(Y=1\mid X\in\calC^\uparrow)$.  
For the upper bound of $\qdn$, we apply the same procedure considering $\mathcal C^\downarrow = \{\mathbf x : f^\star(\mathbf x) < \tfrac12 \le f_0(\mathbf x)\}$, $g_+ \coloneqq \indicator{f_0(X) \ge \tfrac12} \in \calG$ and the level set $I_{-} = \indicator{f^\star(X) < \tfrac12} \in \calI$, giving $\indicator{X \in \calC^\downarrow} = g_+\,I_-$.  
\end{proof}
We can now conclude the proof of the main statement. From~\cref{lemma:lemma2}, multicalibration on the up-flip set gives: 
\[
\bbE[Y\,\indicator{X \in \calC^\uparrow}] \ge \bbE[f^\star(X)\,\indicator{X \in \calC^\uparrow}] - \alpha\,\tau^\uparrow, \qquad \bbE[Y\,\indicator{X \in \calC^\downarrow}] \le \bbE[f^\star\,\indicator{X \in \calC^\downarrow}] + \alpha\,\tau^\downarrow
\]
Using $\bbE[Y\,\indicator{X \in \calC^\uparrow}] = \pup \qup$ and $\bbE[\indicator{X \in \calC^\uparrow}] = \pup$, the first term of \cref{lemma:lemma1} can be rewritten as:
\begin{align}
  \pup(2\qup - 1)
  & =\; 2\,\bbE[Y\,\indicator{X \in \calC^\uparrow}] - \pup \notag\\ 
  & \ge\; 2\,\bbE[f^\star(X)\,\indicator{X \in \calC^\uparrow}] - \pup - 2\alpha\,\tau^\uparrow \notag\\
  & =\; \bbE\bigl[(2f^\star(X) - 1)\,\indicator{X \in \calC^\uparrow}\bigr] - 2\alpha\,\tau^\uparrow.
\end{align}
Symmetrically, using $\bbE[Y\,\indicator{X \in \calC^\downarrow}] = \pdn \qdn$ and $\bbE[\indicator{X \in \calC^\downarrow}] = \pdn$, the second term is:
\begin{align}
  \pdn(1 - 2\qdn)
  & =\; \pdn - 2\,\bbE[Y\,\indicator{X \in \calC^\downarrow}] \notag \\
  &\ge\; \pdn - 2\,\bbE[f^\star(X)\,\indicator{X \in \calC^\downarrow}] - 2\alpha\,\tau^\downarrow \notag \\
  &=\; \bbE\bigl[(1 - 2f^\star(X))\,\indicator{X \in \calC^\downarrow}\bigr] - 2\alpha\,\tau^\downarrow.
\end{align}

Substituting these quantities into~\cref{lemma:lemma1},
\begin{equation}
  \mathrm{Acc}(f^\star(X)) - \mathrm{Acc}(f_0)
  \;\ge\;
  \bbE\bigl[(2f^\star(X) - 1)\,\indicator{X \in \calC^\uparrow}\bigr] + \bbE\bigl[(1 - 2f^\star(X))\,\indicator{X \in \calC^\downarrow}\bigr]
  \;-\; 2\alpha\,(\tau^\uparrow + \tau^\downarrow).
  \label{eq:substitution}
\end{equation}

On $\calC^\uparrow$ we have $f^\star(X) \ge \tfrac12$, so $2f^\star(X) - 1 = \lvert 2f^\star(X) - 1\rvert \ge 0$; on $\calC^\downarrow$ we have $f^\star(X) < \tfrac12$, so $1 - 2f^\star(X) = \lvert 2f^\star(X) - 1\rvert > 0$. Moreover, since the two sets are disjoint, $\indicator{X \in \calC^\uparrow} + \indicator{\calC^\downarrow} = \indicator{X \in \calC^\uparrow \cup \calC^\downarrow}$. Therefore
\[
  \bbE\Biggl[(2f^\star(X) - 1)\,\indicator{X \in \calC^\uparrow}\Biggr] + \bbE\Biggl[(1 - 2f^\star(X))\,\indicator{X \in \calC^\downarrow}\Biggr]
  \;=\; \bbE\Biggl[\Bigl| 2f^\star(X) - 1\Bigr|\,\indicator{X \in \calC^\uparrow \cup \calC^\downarrow}\Biggr]
  \;=\; \Gamma
\]
with $\Gamma \geq 0$ as the expectation of a nonnegative integrand. If we then combine these with \cref{eq:substitution}, we obtain:
\[
\mathrm{Acc}(f^\star(X)) - \mathrm{Acc}(f_0)
  \;\ge\;
  \Gamma \;-\; 2\alpha\,(\tau^\uparrow + \tau^\downarrow).
\]

Because $\Gamma \ge 0$, the right-hand side is positive whenever $\Gamma > 2\alpha(\tau^\uparrow + \tau^\downarrow)$, giving $\mathrm{Acc}(f^\star(X)) > \mathrm{Acc}(f_0)$. In case of perfect multicalibration $\alpha \to 0$, the bound reduces to $\mathrm{Acc}(f^\star(X)) - \mathrm{Acc}(f_0) \ge \Gamma$, strict in presence of $\Gamma > 0$.
\end{proof}

\section{Implementation Details}
\label[appendix]{app:impl_details}
In principle, \method is
agnostic to the choice of multicalibration strategy. In practice, however, we rely on \mcgrad~\citep{DBLP:conf/kdd/TaxPLHKOVA26} as our default multicalibrator because it can automatically search the supplied feature space and thereby uncover undocumented, latent sources of bias. We note that the guarantee in~\cref{thm:acc-improve} additionally requires the multicalibrator to control violations defined jointly by confidence intervals and the base-decision groups $g_-$ and $g_+$. This requirement is naturally satisfied by \mcgrad, which searches for calibration violations using both the supplied features and the base confidence score.
To exploit \mcgrad's ability to discover groups automatically, we encode both candidate responses using the open source \texttt{Muse-Glimmer-30B}~\citep{DBLP:journals/corr/abs-2606-12429} and denote their embeddings with \(e_1, e_2 \in \mathbb{R}^d\). 
Following previous work~\citep{DBLP:journals/corr/abs-2606-19057}, we represent each comparison through the Principal Component Analysis (PCA) of the difference embedding,
\ie, 
\( 
E(X) := {\mathrm{PCA}}(e_1-e_2).
\) 
Such contrastive representation suppresses components shared by the two responses and emphasizes their relative attributes, which are more directly aligned with the structure of pairwise biases. When PCA is then applied to the difference embedding, the retained directions summarize variation in the relation between the two candidates rather than variation in their absolute representation.  
This choice is also consistent with \cref{thm:residual_alignment}: the multicalibrator can only control biases whose error pattern is captured by the supplied representation. A contrastive representation is therefore desirable as it better captures relative differences between the candidate responses.

\paragraph{Hyperparameter tuning.} 
We perform five-fold cross validation. For each held-out fold, the remaining four folds form the calibration set on which we tune and fit \method and the baselines.  
We assign all comparisons that share a question identifier or an identical prompt to the same fold, so that no prompt is seen both during fitting and at test time. 
Hyperparameters are tuned as follows for \method and all baselines. 
We split the calibration set at random into 80\% for training and 20\% for validation, again keeping all comparisons of a question on the same side. For each method, dataset and held-out fold, we run 40 trials of Bayesian optimisation starting from the method's default configuration, score each trial by its validation log loss, and refit the best configuration on the whole calibration set. 
Specifically, for \method we tune jointly the number $k \in \{4, 8, 16, 32\}$ of principal components retained in $E(X)$ with the hyperparameters of \mcgrad~\citep{DBLP:conf/kdd/TaxPLHKOVA26}, using its own tuning routine over the default search space. 
For the tunable baselines we use Optuna~\citep{DBLP:conf/kdd/AkibaSYOK19} over the following hyperparamters (all searched on a log scale, with defaults in parentheses): for \texttt{CalibraEval}, we tune the anti-collapse weight $\lambda \in [0.01, 0.9]$ ($0.1$), the learning rate in $[10^{-3}, 0.3]$ ($0.05$) and the number of iterations in $\{50, \dots, 2000\}$ ($500$); for \texttt{LenControl}, the inverse $L_2$ strength $C \in [10^{-3}, 10^{3}]$ ($1$) and the minimum model frequency in $\{1, \dots, 60\}$ ($10$) ; for \texttt{IGLB}, the number of bins $M \in \{3, \dots, 100\}$ ($20$), the maximum number of iterations in $\{100, \dots, 2000\}$ ($500$) and the tolerance $\epsilon \in [10^{-4}, 0.2]$ ($0.01$). The other baselines do not have any hyperparameter to tune.
The selected configurations are fold-specific and available in the code repository.

\section{Data}
\label[appendix]{app:data}
We use five pairwise evaluation benchmarks from the literature. \texttt{RewardBench} (\texttt{RWB})~\citep{DBLP:conf/naacl/LambertPMMLCDKZCSH25} consists of $3.0$k prompt-chosen-rejected trios belonging to four different macro-categories: \textit{chat} (strong vs. weak answers), \textit{chat-hard} (tricky and adversarial answers), \textit{safety} (refuse dangerous content and avoid over-refusals) and \textit{reasoning} (coding and math problems). Answers are either LLM-generated (coming from 44 different models) 
or crafted manually, and the ground truth label is a manually verified correctness judgment rather than a human vote.
\texttt{MTBench} (\texttt{MTB})~\citep{zheng2023judging} consists of $3.3$k human votes on pairwise comparisons over 80 questions from eight categories (writing, roleplay, reasoning, math, coding, extraction, stem, humanities), with answers generated by six different LLMs (GPT-4, GPT-3.5, Claude-v1, Vicuna-13B, Alpaca-13B, and LLaMA-13B) and judged by 65 human annotators.
We aggregate the individual votes into one comparison per (question, model pair) by majority vote. We further discard ties, i.e., cases where the annotators split evenly or could not identify a preferable option, and second-turn data, i.e., each question has two conversational turns, but we keep only the first one since each model's second-turn answer is conditioned on its own first turn answer. This leaves 709 pairwise comparisons.
\texttt{ArenaExpl} (\texttt{AEX})~\citep{DBLP:conf/icml/chatarena} is a corpus of $106$k pairwise comparisons extracted from conversational data where users write the prompt, receive two LLM-generated answers (coming from 55 different models), and act as the judge by expressing a preference. Prompts are written in different languages. We again discard ties (\ie cases where the user does not prefer any of the two options, or where both responses are bad) and keep only single-turn conversations. This leaves $53.8$k pairwise comparisons.
Similarly, \texttt{PPE}~\citep{DBLP:conf/iclr/Frick25} contains $16$k prompts in different languages of users who cast the preference vote for one out of two options generated by two LLMs from a pool of 20 models. 
After discarding ties, the corpus amounts to $10.2$k pairwise comparisons.
Finally, \texttt{PKU}~\citep{DBLP:conf/acl/JiPKU} consists of $82.1$k pairwise comparisons over safety-oriented prompts generated combining human demonstrations with LLMs. Each pair carries two human preference labels: which response is safer and which is more helpful; we use the latter as ground truth. No ties are allowed. Both responses are generated by an LLM from the Alpaca-family (3 models in total). We use the \textit{train} split of the original dataset which consists of $73.9$k pairwise comparisons.

Note that family bias can only be measured on comparisons where exactly one of the two answers comes from the judge's own model family. For \texttt{Muse-Glimmer-30B} this is the case in all datasets except \texttt{PKU}, where both answers of a pair are generated by the same model. We consider answers from Llama models as the same family as Muse, since both are developed by Meta. For \texttt{Qwen2.5-7B-Instruct}, only \texttt{PPE} and \texttt{AEX} contain answers generated by the same model family, so family is not measurable on \texttt{MTB}, \texttt{RWB} and \texttt{PKU}.

\section{Metrics}
\label[appendix]{app:metrics}

\paragraph{ECCE.}  Expected Cumulative Calibration Error ($\operatorname{ECCE}$) \citep{tygert2023ecce} was first introduced as a reliable parameter free alternative to binning-based calibration metrics. $\operatorname{ECCE}$ computes the cumulative calibration residuals:
\[
C_k(f)
:=
\frac{1}{n}
\sum_{i=1}^{k}
\left(y_{(i)}-f(\mbx_{(i)})\right),
\qquad
k=1,\ldots,n
\]
where
\(
f(\mbx_{(1)})\leq f(\mbx_{(2)})\leq\cdots\leq f(\mbx_{(n)})
\) denote the ordered values of $f$ over a dataset.

\begin{definition}
    $\operatorname{ECCE\text{-}R}$ is the range of the empirical cumulative calibration error defined as:
    \[    
    \operatorname{ECCE\text{-}R}
    :=
    \max_{0\leq k\leq n}C_k(f)
    -
    \min_{0\leq k\leq n}C_k(f).
    \]
\end{definition}
Thus, $\operatorname{ECCE\text{-}R}$ measures the full range of the cumulative residual curve and therefore captures the largest cumulative calibration error over any contiguous score interval
\citep{tygert2023ecce}.

\paragraph{MCE.} The \textit{Multicalibration} property of a model is evaluated via the multicalibration Error ($\operatorname{MCE}$). This relies on the Expected Cumulative Calibration Error ($\operatorname{ECCE}$) \citep{tygert2023ecce}. The Multicalibration Error ($\operatorname{MCE}$) is then the largest standardized ECCE-R across all considered groups and confidence bands~\citep{guy2025measuring,DBLP:conf/kdd/TaxPLHKOVA26}.

\begin{definition:mce}[Multicalibration Error] 
Consider a collection $\mathcal{G}$ of measurable group indicators $g:\mathcal{X}\rightarrow\{0,1\}$. For each $g\in\mathcal{G}$, let
\(
\mathcal{D}_g
:=
\{(\mbx_i,y_i)\}_{i=1}^{n_g}
\)
denote the sample restricted to instances satisfying $g(\mbx_i)=1$,
and let $\operatorname{ECCE\text{-}R}_g$ denote the ECCE-R computed
restricting the sample to $\mathcal{D}_g$. Under the null hypothesis of perfect calibration, define the corresponding noise scale as 
\(
\tau_g(f)
:=
\frac{1}{n_g}
\sqrt{
\sum_{i=1}^{n_g}
f(x_i)\bigl(1-f(x_i)\bigr)
}.
\)
The empirical Multicalibration Error of $f$ over
$\mathcal{G}$ is 
\[
\operatorname{MCE}
:=
\max_{g\in\mathcal{G}}
\frac{
\operatorname{ECCE\text{-}R}_g(f)
}{
\tau_g(f)
}
\]
\end{definition:mce}

\noindent The normalization accounts for the different sampling variability of groups of different sizes, while the outer maximum measures the strongest statistical evidence of group-wise miscalibration. 
\paragraph{Log loss.}
We additionally report the binary log loss as a measure of probabilistic predictive performance:
\[
\operatorname{LogLoss}(f)
:=
-\frac{1}{n}
\sum_{i=1}^{n}
\left[
y_i\log f(\mbx_i)
+
(1-y_i)\log\bigl(1-f(\mbx_i)\bigr)
\right].
\]
Log loss is a strictly proper scoring rule~\citep{gneiting2007strictly},
meaning that its expected value is uniquely minimized when
\(
f(\mbx)=\bbP(Y=1\mid X=\mbx).
\)
It therefore rewards probabilistic predictions that are both well calibrated
and informative, while strongly penalizing confident incorrect predictions. Accordingly, we use log loss as a measure of the overall quality of the judge's probabilistic predictions, with lower values indicating better
predictive performance. A value of zero is attained only when every observed outcome is assigned probability one for the correct class.

\section{Additional Experiments}
\label[appendix]{app:add_experiments}

\subsection{Results on specified and unspecified groups}
\label[appendix]{app:unspec}
\begin{figure}[h]
    \centering
    \includegraphics[width=1\linewidth]{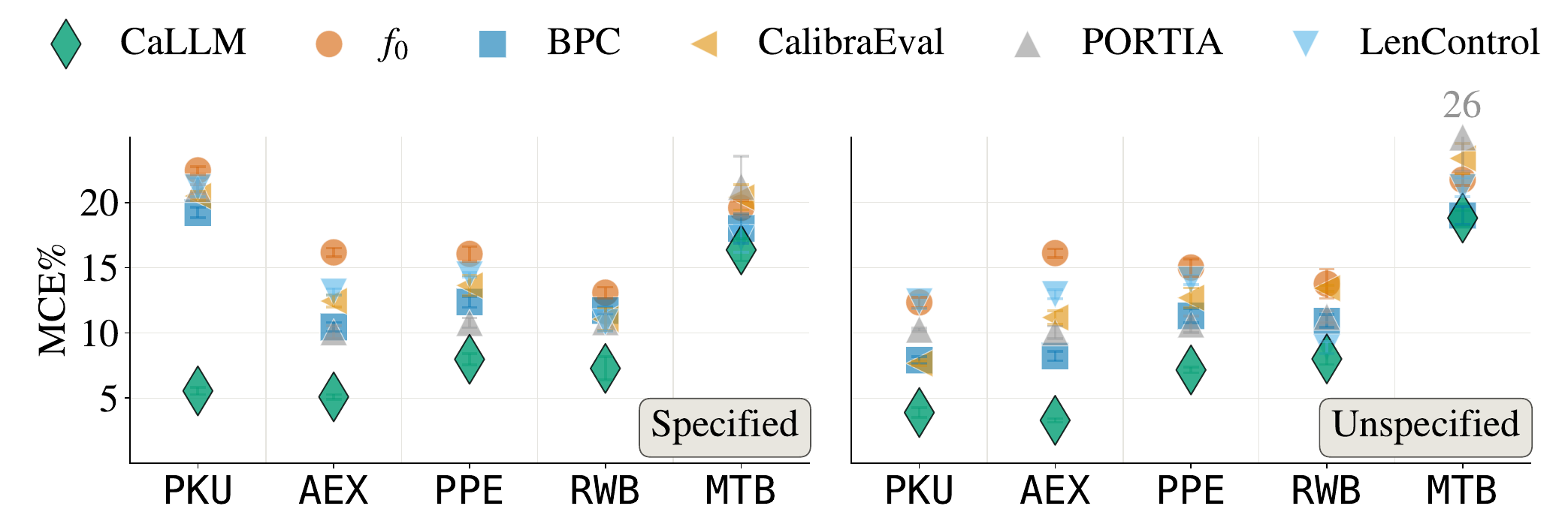}
    \caption{Average per-axis $\text{MCE}_{\%}$ ($\pm$ s.e.) on the \emph{specified} and \emph{unspecified} groups for each method across datasets.}
    \label{fig:APP_specified_unspecified}
\end{figure}
We complement the per-bias analysis of Q1 with two additional analyses. First, we consider the \emph{specified} biases jointly (\ie position, verbosity and family): the MCE is computed over every position, verbosity and family group together with all their pairwise and three-way intersections. Second, we consider \emph{unspecified} groups. Following \citet{DBLP:conf/kdd/TaxPLHKOVA26}, we pass the concatenated embeddings $[e_Q, e_1, e_2]$ of the prompt and the two responses to the MCE as numerical segment features to create $1000$ groups. These groups depend on neither labels nor predictions and are identical for all methods.

\paragraph{Accounting for interactions between biases increases \method's advantage.}
\Cref{fig:APP_specified_unspecified} (left) reports the joint $\text{MCE}_{\%}$ over the \emph{specified} groups and their intersections. \method attains the lowest $\text{MCE}_{\%}$ on all five datasets, whereas it is the best method on 11 of 14 (dataset-bias) pairs when each bias is scored alone. A baseline that targets one specific bias does not mitigate other biases as well; \eg on \texttt{RWB}, \texttt{BPC} has the lowest position error ($6.10 \pm 0.26$) but a joint error of $11.71 \pm 0.27$, whereas \method reaches $7.27 \pm 0.89$.

\paragraph{\method also mitigates the bias on unspecified groups.} \Cref{fig:APP_specified_unspecified} (right) reports the $\text{MCE}_{\%}$ on the \emph{unspecified} groups. \method attains the lowest $\text{MCE}_{\%}$ on four of five datasets and lowers that of $f_0$ by a median of ${\approx}52\%$, more than twice the best baseline (\texttt{BPC}, ${\approx}25\%$). Against the best baseline on each dataset, it lowers the error by $60\%$ on \texttt{AEX}, $49\%$ on \texttt{PKU}, $33\%$ on \texttt{PPE}, and  $10\%$ on \texttt{RWB}. On \texttt{MTB}, its gain over \texttt{BPC} ($1\%$) is within one standard error. Thus, \method is also effective at mitigating the bias for unspecified groups.

\subsection{Judge ablation: results with a different judge}
\label[appendix]{app:judge_ablation}
We repeat the experiments from Section~\ref{sec:experiments} using
\texttt{Qwen~2.5-7B-Instruct} as the judge, taking its verbalized confidence as $f_0$.
\begin{figure}[h]
    \centering
    \includegraphics[width=1\linewidth]{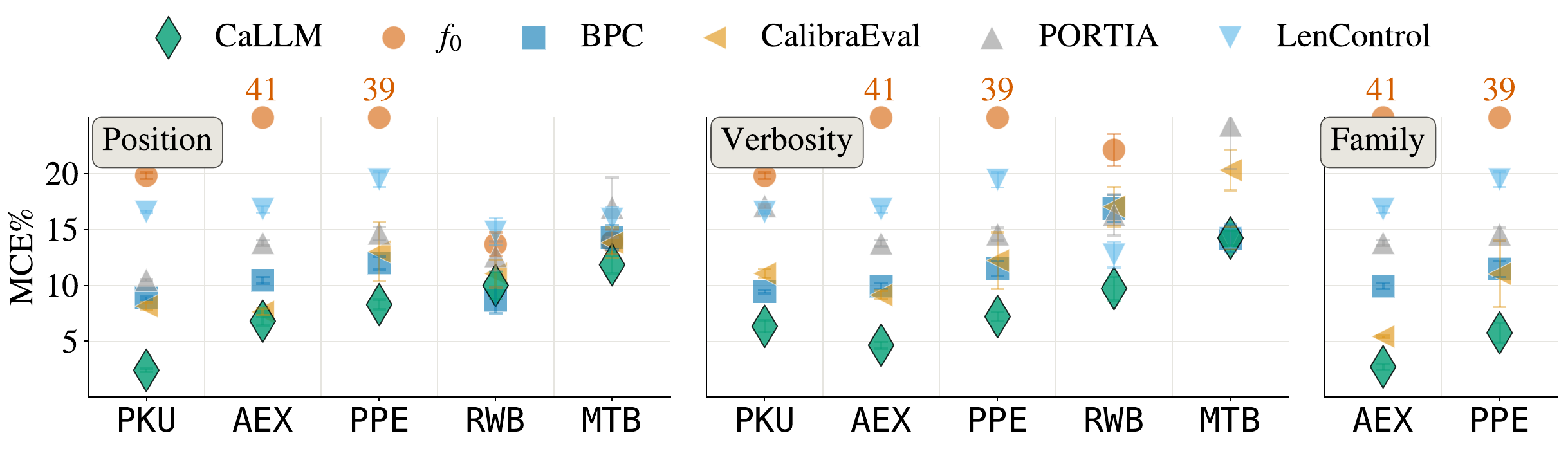}
    \caption{Average per-group $\text{MCE}_\%$ ($\pm$ s.e.) for \method and the mitigation baselines across datasets with \texttt{Qwen~2.5-7B-Instruct} used as the judge. The self-enhancement axis contains \texttt{AEX} and \texttt{PPE} only: \texttt{PKU}, \texttt{RWB} and \texttt{MTB} do not contain \texttt{qwen}-family responses. }
    \label{fig:APP_Q1_qwenjudge}
\end{figure}

\textbf{Q1. The bias mitigation transfers to a different judge.} 
\Cref{fig:APP_Q1_qwenjudge} reports the per-group $\text{MCE}_\%$ ($\pm$ s.e.) for \method and the \cref{fig:Q1} baselines when \texttt{Qwen~2.5-7B-Instruct} is used as the judge. \method attains the lowest $\text{MCE}_\%$ on $10$ of the $12$ (dataset, bias group) pairs: four out of five on \textit{Position}, four out of five on \textit{Verbosity}, and two out of two on \textit{Family}. Even when the verbalized confidences come from a different judge, \method still mitigates the bias effectively.

\begin{table}[]
\scriptsize
\caption{$\text{MCE}_\%$ per bias group ($\pm 1$ s.e.) for \method and
    \methodG {} using the \texttt{Qwen~2.5-7B-Instruct} judge. The
    \emph{family} axis contains results for \texttt{AEX} and \texttt{PPE} only:
    \texttt{PKU}, \texttt{RWB} and \texttt{MTB} contain no
    \texttt{qwen}-family response.}
\setlength{\tabcolsep}{4pt}
    \centering
    \scalebox{0.85}{
    \begin{tabular}{l|cc|cc|cc|cc|cc}
\toprule
 & \multicolumn{2}{c|}{PKU} & \multicolumn{2}{c|}{AEX} & \multicolumn{2}{c|}{PPE} & \multicolumn{2}{c|}{RWB} & \multicolumn{2}{c}{MTB} \\
 & \method & \methodG & \method & \methodG & \method & \methodG & \method & \methodG & \method & \methodG \\
\midrule
\textit{Position} & \textbf{2.4 $\pm$ 0.2} & \textbf{2.4 $\pm$ 0.2} & 6.8 $\pm$ 0.4 & \textbf{2.4 $\pm$ 0.3} & 8.3 $\pm$ 0.5 & \textbf{3.7 $\pm$ 0.5} & 10.0 $\pm$ 1.2 & \textbf{7.9 $\pm$ 1.1} & \textbf{11.8 $\pm$ 0.8} & 12.0 $\pm$ 0.8 \\
\textit{Verbosity} & 6.3 $\pm$ 0.5 & \textbf{2.8 $\pm$ 0.2} & 4.6 $\pm$ 0.3 & \textbf{2.3 $\pm$ 0.1} & 7.2 $\pm$ 0.4 & \textbf{4.0 $\pm$ 0.3} & 9.7 $\pm$ 1.0 & \textbf{7.4 $\pm$ 0.8} & 14.4 $\pm$ 0.7 & \textbf{13.5 $\pm$ 1.1} \\
\textit{Family} & -- & -- & 2.7 $\pm$ 0.3 & \textbf{2.1 $\pm$ 0.2} & 5.7 $\pm$ 0.9 & \textbf{4.7 $\pm$ 0.9} & -- & -- & -- & -- \\
\bottomrule
\end{tabular}}
    \label{tab:APP_Q2_qwenjudge}
\end{table}
\textbf{Q2. Declaring the bias groups helps regardless of the judge.}
Table~\ref{tab:APP_Q2_qwenjudge} repeats the \method vs. \method$\!+g$ comparison using the \texttt{Qwen~2.5-7B-Instruct} judge to obtain the verbalized confidences. As in Q2, giving \method the explicit bias axes further reduces the bias on most (dataset, bias
group) pairs.

\begin{table}[]
\scriptsize
\caption{Average accuracy and log-loss ($\pm 1$ s.e.) across datasets using \texttt{Qwen~2.5-7B-Instruct} as judge.}
\setlength{\tabcolsep}{4pt}
    \centering
    \scalebox{0.85}{
    \begin{tabular}{l|cc|cc|cc|cc|cc}
\toprule
 & \multicolumn{2}{c|}{PKU} & \multicolumn{2}{c|}{AEX} & \multicolumn{2}{c|}{PPE} & \multicolumn{2}{c|}{RWB} & \multicolumn{2}{c}{MTB} \\
 & Acc & $\mathcal{L}_{\log}$ & Acc & $\mathcal{L}_{\log}$ & Acc & $\mathcal{L}_{\log}$ & Acc & $\mathcal{L}_{\log}$ & Acc & $\mathcal{L}_{\log}$ \\
\midrule
\method & \textbf{.68 $\pm$ .01} & \textbf{.60 $\pm$ .01} & \textbf{.61 $\pm$ .01} & \textbf{.65 $\pm$ .01} & \textbf{.61 $\pm$ .01} & \textbf{.66 $\pm$ .01} & \textbf{.84 $\pm$ .01} & \textbf{.37 $\pm$ .01} & \textbf{.80 $\pm$ .01} & \textbf{.45 $\pm$ .02} \\
\midrule
$f_0$ & .66 $\pm$ .01 & .72 $\pm$ .01 & .58 $\pm$ .01 & .88 $\pm$ .01 & .59 $\pm$ .01 & .87 $\pm$ .01 & .80 $\pm$ .01 & .50 $\pm$ .02 & \textbf{.80 $\pm$ .01} & .51 $\pm$ .02 \\
\texttt{BPC} & .66 $\pm$ .01 & .65 $\pm$ .01 & .59 $\pm$ .01 & .71 $\pm$ .01 & .60 $\pm$ .01 & .72 $\pm$ .01 & .80 $\pm$ .01 & .42 $\pm$ .01 & \textbf{.80 $\pm$ .01} & .46 $\pm$ .02 \\
\texttt{CalibraEval} & .66 $\pm$ .01 & .63 $\pm$ .01 & .59 $\pm$ .01 & .68 $\pm$ .01 & .60 $\pm$ .01 & .69 $\pm$ .01 & .80 $\pm$ .01 & .43 $\pm$ .01 & \textbf{.80 $\pm$ .01} & .47 $\pm$ .02 \\
\texttt{PORTIA} & .59 $\pm$ .01 & .71 $\pm$ .01 & .59 $\pm$ .01 & .75 $\pm$ .01 & .60 $\pm$ .01 & .74 $\pm$ .01 & .76 $\pm$ .01 & .50 $\pm$ .01 & .70 $\pm$ .03 & .59 $\pm$ .06 \\
\texttt{LenControl} & .66 $\pm$ .01 & .64 $\pm$ .01 & .59 $\pm$ .01 & .69 $\pm$ .01 & .60 $\pm$ .01 & .69 $\pm$ .01 & .80 $\pm$ .01 & .45 $\pm$ .01 & \textbf{.80 $\pm$ .01} & .46 $\pm$ .02 \\

\bottomrule
\end{tabular}}
    \label{tab:APP_Q3_qwenjudge}
\end{table}
\textbf{Q3. \method improves the predictive performance also when \texttt{Qwen~2.5-7B-Instruct} is used as judge.} Table~\ref{tab:APP_Q3_qwenjudge} reports the average accuracy and log-loss for \method and the baselines on each dataset. \method improves on every baseline in both metrics and on every dataset, except on \texttt{MTB} where it ties with other baselines. Multicalibrating \texttt{Qwen~2.5-7B-Instruct} therefore improves predictive performance, and the main empirical findings therefore also hold for this judge.

\subsection{Embedder ablation: results using a separate open-source embedder}
\label[appendix]{app:embedder_ablation}
We refit \method unchanged with the embeddings obtained from an off-the-shelf embedder,
\ie \texttt{Qwen3-Embedding-0.6B}, to show that the method also applies to judges whose
hidden states are not accessible, such as API-only models.

\begin{figure}[h]
    \centering
    \includegraphics[width=1\linewidth]{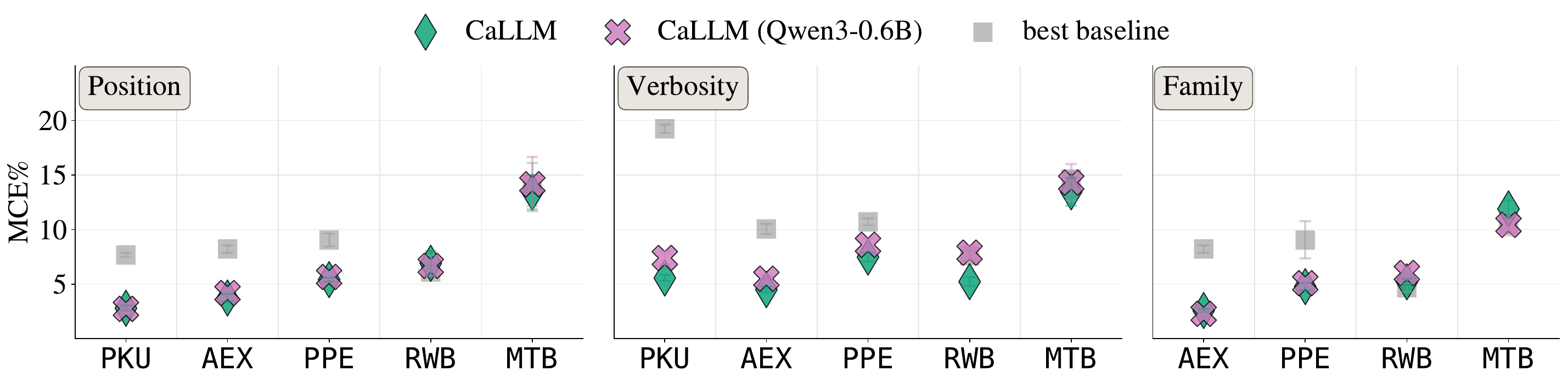}
    \caption{Average per-group $\text{MCE}_\%$ ($\pm 1$ standard error) for \method (embeddings from the judge's last hidden state), \method{\tt (Qwen3-0.6B)} (embeddings from \texttt{Qwen3-Embedding-0.6B}) and, as a reference, the \emph{best} of the \cref{fig:Q1} baselines across datasets. Family is undefined on \texttt{PKU}, since it contains no response from the judge's own family.}
    \label{fig:APP_Q1_qwenembedder}
\end{figure}
\textbf{\method does not need the judge's hidden states to mitigate the bias.} We compare \method (\ie the embeddings are the judge's last hidden state) with its ablated variant \method{\tt (Qwen3-0.6B)}, which takes them from the off-the-shelf sentence embedder (\texttt{Qwen3-Embedding-0.6B}). As a reference we additionally report the \emph{best} of the \cref{fig:Q1} baselines on that (dataset, bias group) pair.

\Cref{fig:APP_Q1_qwenembedder} reports the per-group $\text{MCE}_\%$ ($\pm$ standard error) for \method, its ablated variant and the \emph{best} baseline. On \textit{Position} and \textit{Family} the two variants perform very similarly on nearly every dataset, and where they differ most (\ie \textit{Family} on \texttt{MTB}) the ablated variant is the better one ($10.4\pm0.8$ vs.\ $11.9\pm0.8$). \textit{Verbosity} degrades on all five datasets, plausibly because answer length is a property that the judge's own states encode directly. Even so, \method{\tt (Qwen3-0.6B)} still beats the \emph{best baseline} on $10$ of the $14$ (dataset, bias group) pairs, including four of the five on verbosity. As for \method, its losses concentrate on the two smallest datasets: \texttt{RWB} (all three groups, by at most $1.3$ points) and \textit{Position} on \texttt{MTB}, where it is behind the best baseline by only $0.03$ points. These results show that \method's bias reduction does not require access to the judge's internal representations.

\begin{table}[]
\scriptsize
\setlength{\tabcolsep}{4pt}
    \centering
     \caption{$\text{MCE}_\%$ per bias group ($\pm 1$ standard error) for \method and \methodG{} across datasets using \texttt{Qwen3-Embedding-0.6B} embeddings. Self-enhancement is undefined on \texttt{PKU} because the dataset does not contain \texttt{muse}-family responses.}
    \scalebox{0.85}{
    \begin{tabular}{l|cc|cc|cc|cc|cc}
\toprule
 & \multicolumn{2}{c|}{PKU} & \multicolumn{2}{c|}{AEX} & \multicolumn{2}{c|}{PPE} & \multicolumn{2}{c|}{RWB} & \multicolumn{2}{c}{MTB} \\
 & \method & \methodG & \method & \methodG & \method & \methodG & \method & \methodG & \method & \methodG \\
\midrule
\textit{Position} & 2.7 $\pm$ 0.2 & \textbf{2.2 $\pm$ 0.1} & 4.2 $\pm$ 0.3 & \textbf{2.1 $\pm$ 0.2} & 5.7 $\pm$ 0.3 & \textbf{4.1 $\pm$ 0.2} & 6.7 $\pm$ 1.0 & \textbf{5.6 $\pm$ 0.5} & 14.2 $\pm$ 2.5 & \textbf{11.9 $\pm$ 1.5} \\
\textit{Verbosity} & 7.4 $\pm$ 0.6 & \textbf{2.4 $\pm$ 0.1} & 5.5 $\pm$ 0.3 & \textbf{2.3 $\pm$ 0.2} & 8.6 $\pm$ 0.5 & \textbf{4.3 $\pm$ 0.1} & 7.9 $\pm$ 0.7 & \textbf{5.2 $\pm$ 0.5} & 14.3 $\pm$ 1.7 & \textbf{11.7 $\pm$ 1.4} \\
\textit{Family} & -- & -- & 2.3 $\pm$ 0.1 & \textbf{2.2 $\pm$ 0.1} & \textbf{5.1 $\pm$ 0.3} & \textbf{5.1 $\pm$ 0.3} & \textbf{6.0 $\pm$ 0.6} & 6.6 $\pm$ 0.7 & \textbf{10.4 $\pm$ 0.8} & 11.1 $\pm$ 1.7 \\
\bottomrule
\end{tabular}} \label{tab:APP_Q2_qwenembedder}
\end{table}
\textbf{Declaring the bias groups still helps under an off-the-shelf embedder.} \Cref{tab:APP_Q2_qwenembedder} repeats the \method vs. \methodG~comparison of Q2
with the embeddings taken from \texttt{Qwen3-Embedding-0.6B}. Handing the reference the
three group indicators lowers the $\text{MCE}_\%$ on every dataset for \textit{Position} and
\textit{Verbosity}, and the gain is largest on verbosity: this falls from $7.4\pm0.6$ to $2.4\pm0.1$ on \texttt{PKU} and from
$8.6\pm0.5$ to $4.3\pm0.1$ on \texttt{PPE}. For \textit{Family} the change is
always within one standard error, in either direction. The value of declaring a bias axis
is therefore unchanged by the embedder.

\begin{table}[]
\scriptsize
\setlength{\tabcolsep}{4pt}
    \centering
\caption{Average accuracy and log-loss ($\pm$ s.e.) for \method and its ablated variant \method{\tt (Qwen3-0.6B)} across datasets. \emph{best baseline} is the strongest of the \cref{fig:Q1} baselines on each (dataset, metric) pair.}
    \scalebox{0.85}{
    \begin{tabular}{l|cc|cc|cc|cc|cc}
\toprule
 & \multicolumn{2}{c|}{PKU} & \multicolumn{2}{c|}{AEX} & \multicolumn{2}{c|}{PPE} & \multicolumn{2}{c|}{RWB} & \multicolumn{2}{c}{MTB} \\
 & ACC & $\mathcal{L}_{\log}$ & ACC & $\mathcal{L}_{\log}$ & ACC & $\mathcal{L}_{\log}$ & ACC & $\mathcal{L}_{\log}$ & ACC & $\mathcal{L}_{\log}$ \\
\midrule
\method & \textbf{.69 $\pm$ .01} & \textbf{.57 $\pm$ .01} & \textbf{.68 $\pm$ .01} & \textbf{.60 $\pm$ .01} & \textbf{.67 $\pm$ .01} & \textbf{.61 $\pm$ .01} & \textbf{.93 $\pm$ .01} & \textbf{.21 $\pm$ .01} & \textbf{.81 $\pm$ .01} & \textbf{.41 $\pm$ .02} \\
\method{} (Qwen3-0.6B) & \textbf{.69 $\pm$ .01} & .58 $\pm$ .01 & .66 $\pm$ .01 & .61 $\pm$ .01 & .66 $\pm$ .01 & .62 $\pm$ .01 & .91 $\pm$ .01 & .27 $\pm$ .01 & \textbf{.81 $\pm$ .01} & .42 $\pm$ .02 \\
\emph{best baseline} & .64 $\pm$ .01 & .65 $\pm$ .01 & .65 $\pm$ .01 & .69 $\pm$ .01 & .66 $\pm$ .01 & .66 $\pm$ .01 & .88 $\pm$ .01 & .29 $\pm$ .01 & .81 $\pm$ .01 & .43 $\pm$ .02 \\
\bottomrule
\end{tabular}}
\label{tab:APP_Q3_qwenembedder}
\end{table}
\textbf{Using an off-the-shelf embedder still improves the predictive performance against the baselines.} Table~\ref{tab:APP_Q3_qwenembedder} reports accuracy and log-loss for \method, \method{\tt (Qwen3-0.6B)} and the \emph{best} Q1 baseline for each (dataset, metric) pair. The two variants are close on four of the five datasets, while on \texttt{RWB} the judge's own states buy a real margin ($.927$ vs. $.911$ accuracy, $.203$ vs. $.264$ log-loss). Crucially, both variants beat the strongest baseline on all ten (dataset, metric) pairs, so \method still outperforms every baseline even when using an off-the-shelf embedder.

\subsection{Ablation study: multicalibration method}
\label[appendix]{app:ablate_multi}
\begin{figure}[h]
    \centering
    \includegraphics[width=1\linewidth]{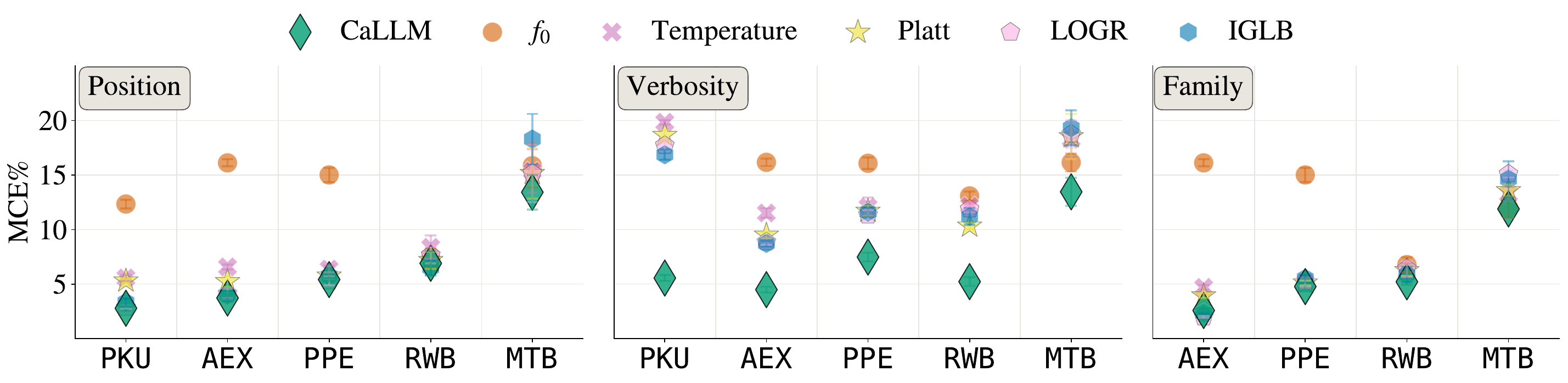}
\caption{Average per-group $\text{MCE}_\%$ ($\pm 1$ s.e.) across datasets for \method and the calibration and multicalibration baselines. Family is undefined on \texttt{PKU}, since it contains no response from the judge's own family.}
    \label{fig:APP_calibration_ablation}
\end{figure}
In this experiment, we evaluate alternative multicalibration strategies, all fitted to the same base scores \(f_0\), validating our choice of \mcgrad as multicalibrator (still denoted as \method in both tables and figures for coherence with the main text). We emphasize that \method is agnostic to the choice of multicalibrator and can be used with any multicalibration procedure. We consider two \emph{multicalibration} baselines, which correct the scores within groups: \texttt{LOGR}, a group-conditional unbiased logistic regression~\citep{DBLP:conf/colt/GopalanKSZ22}, and \texttt{IGLB}, iterative grouped linear binning~\citep{detommaso2024multicalibration}. Like \method, these baselines have no access to the bias groups, so we build their groups with the clustering strategy of \citet{detommaso2024multicalibration}. That is, we embed each comparison (the prompt followed by both responses) with the same encoder as \method, reduce the embedding to 20 dimensions with UMAP, and fit a Gaussian mixture on the reduced embedding together with $f_0$, choosing the number of components in $\{2,\dots,20\}$ by BIC. Each component with at least 30 calibration comparisons defines a group. We also include $f_0$ for comparison with the uncalibrated judge. Finally, we compare with two \emph{global} calibration strategies: temperature scaling~\citep{DBLP:conf/icml/GuoPSW17} and Platt scaling~\citep{platt1999probabilistic}. 

\cref{fig:APP_calibration_ablation} shows the per-group $\text{MCE}_{\%}$ of \method and the other calibration strategies across datasets. 
Correcting the scores within groups decreases the biases: \mcgrad beats both global calibrators on all 14 (dataset--bias) pairs, while \texttt{LOGR} and \texttt{IGLB} fail to beat them mostly on the two smallest datasets, \texttt{RWB} and \texttt{MTB}. \mcgrad is the best on 11 of 14 pairs, with the largest gap on \textit{Verbosity}, where it lowers the error of the best calibrator by at least $27\%$. Indeed, verbosity depends on how the two responses compare, which the  clusters of the other multicalibration approaches are unlikely to isolate; \mcgrad instead searches the contrastive $\phi(X)$ for the groups where the judge is miscalibrated.

We report in~\cref{tab:acc_log_multi} accuracy and log-loss on all five datasets. \mcgrad is the best on both metrics for all datasets with the sole exception of accuracy  in \texttt{MTB}: \mcgrad ties with the two multicalibration strategies \ie, \texttt{LOGR} and \texttt{IGLB}.

\paragraph{}
\begin{table}[]
    \centering
    \setlength{\tabcolsep}{4pt}
    \caption{Average ACC and log-loss $\mathcal{L}_{\log}$ ($\pm$ s.e.) across datasets. \method always matches or improves both metrics over all calibration and multicalibration strategies.}
    \scalebox{0.7}{\begin{tabular}{l|cc|cc|cc|cc|cc}
\toprule
 & \multicolumn{2}{c|}{\texttt{PKU}} & \multicolumn{2}{c|}{\texttt{AEX}} & \multicolumn{2}{c|}{\texttt{PPE}} & \multicolumn{2}{c|}{\texttt{RWB}} & \multicolumn{2}{c}{\texttt{MTB}} \\
 & ACC & $\mathcal{L}_{\log}$ & ACC & $\mathcal{L}_{\log}$ & ACC & $\mathcal{L}_{\log}$ & ACC & $\mathcal{L}_{\log}$ & ACC & $\mathcal{L}_{\log}$ \\
\midrule
\method & \textbf{.69 $\pm$ .01} & \textbf{.57 $\pm$ .01} & \textbf{.68 $\pm$ .01} & \textbf{.60 $\pm$ .01} & \textbf{.67 $\pm$ .01} & \textbf{.61 $\pm$ .01} & \textbf{.93 $\pm$ .01} & \textbf{.21 $\pm$ .01} & \textbf{.81 $\pm$ .01} & \textbf{.41 $\pm$ .02} \\
\midrule
$f_0$ & .63 $\pm$ .01 & .73 $\pm$ .01 & .63 $\pm$ .01 & .76 $\pm$ .01 & .64 $\pm$ .01 & .74 $\pm$ .01 & .86 $\pm$ .01 & .36 $\pm$ .02 & .79 $\pm$ .01 & .49 $\pm$ .02 \\
\texttt{Temperature} & .64 $\pm$ .01 & .63 $\pm$ .01 & .65 $\pm$ .01 & .64 $\pm$ .01 & .65 $\pm$ .01 & .63 $\pm$ .01 & .88 $\pm$ .01 & .30 $\pm$ .01 & .80 $\pm$ .01 & .43 $\pm$ .01 \\
\texttt{Platt} & .64 $\pm$ .01 & .63 $\pm$ .01 & .65 $\pm$ .01 & .64 $\pm$ .01 & .65 $\pm$ .01 & .63 $\pm$ .01 & .88 $\pm$ .01 & .30 $\pm$ .01 & .80 $\pm$ .01 & .43 $\pm$ .01 \\
\texttt{LOGR} & .64 $\pm$ .01 & .62 $\pm$ .01 & .65 $\pm$ .01 & .63 $\pm$ .01 & .65 $\pm$ .01 & .63 $\pm$ .01 & .88 $\pm$ .01 & .30 $\pm$ .01 & .81 $\pm$ .01 & .43 $\pm$ .01 \\
\texttt{IGLB} & .64 $\pm$ .01 & .62 $\pm$ .01 & .64 $\pm$ .01 & .63 $\pm$ .01 & .65 $\pm$ .01 & .63 $\pm$ .01 & .88 $\pm$ .01 & .30 $\pm$ .01 & .81 $\pm$ .01 & .44 $\pm$ .02 \\
\bottomrule
\end{tabular}}
    \label{tab:acc_log_multi}
\end{table}

\subsection{Validating Recoverability}
\label[appendix]{app:recover}
    
In this section, we assess whether \method can recover the input-space groups $\calG$ from the supplied representation $[\phi(X),f_0(X)]$. This condition directly reduces the terms $\varphi$ and $\psi$ in~\cref{thm:residual_alignment} as illustrated in Corollary 1. Specifically, we take $\calH$ to be the tree family used by \mcgrad and measure how accurately it predicts membership in the \textit{position}, \textit{verbosity}, and \textit{family} bias groups. We carry this analysis using \texttt{Muse-Glimmer-30B} as both the judge and the encoder as in~\cref{sec:experiments}.  
\begin{table}[h]
\centering
\caption{Recoverability results across datasets.}
\begin{tabular}{l|ccccc}
\toprule
 & PKU & AEX & PPE & RWB & MTB \\
\midrule
Position & 0.50 $\pm$ 0.01 & 0.76 $\pm$ 0.01 & 0.76 $\pm$ 0.01 & 0.93 $\pm$ 0.01 & 0.66 $\pm$ 0.05 \\
Verbosity & 0.92 $\pm$ 0.01 & 0.90 $\pm$ 0.01 & 0.89 $\pm$ 0.01 & 0.95 $\pm$ 0.01 & 0.94 $\pm$ 0.02 \\
Family & - & 0.98 $\pm$ 0.01 & 0.98 $\pm$ 0.01 & 0.99 $\pm$ 0.01 & 0.98 $\pm$ 0.02 \\
\bottomrule
\end{tabular}
\label{tab:recoverability}
\end{table}

Results are reported in~\cref{tab:recoverability}. Across the 14 available dataset-group pairs, the bias groups are generally well recoverable from the representation, with median accuracies of $\approx0.76$ for \textit{Position}, $\approx0.92$ for \textit{Verbosity}, and $\approx0.98$ for \textit{Family}. The main exception is \textit{Position} on \texttt{PKU}, where recoverability is at chance level ($0.50$). 
Interestingly, \method still substantially reduces the \textit{Position} bias on \texttt{PKU}, and explicitly providing the position groups yields little further improvement. This is consistent with our analysis: poor group recoverability need not prevent bias mitigation when the corresponding judge-error components, \(\lambda\) and \(\chi\), are sufficiently small, allowing larger group-side terms \(\varphi\) and \(\psi\) without inducing large residual bias.
Thus, \method can mitigate position bias even when position membership itself is poorly recoverable.

\subsection{Detailed results}
\label[appendix]{app:detailed_results}

We provide all our results in what follows.
\begin{table}[]
    \centering
    \setlength{\tabcolsep}{4pt}
\caption{Average per-group $\text{MCE}_\%$ and $\text{MCE}_{\sigma}$ ($\pm$ standard error) for \method and the mitigation baselines across datasets. The \emph{family} axis does not show \texttt{PKU} since it does not contain \texttt{muse}-family responses.}
\scalebox{0.8}{\begin{tabular}{ll|ccc|ccc}
\toprule
 &  & \multicolumn{3}{c|}{$\text{MCE}_{\%}$} & \multicolumn{3}{c}{$\text{MCE}_{\sigma}$} \\
\cmidrule(lr){3-5}\cmidrule(lr){6-8}
Data & Method & Position & Verbosity & Family & Position & Verbosity & Family \\
\midrule
\multirow{6}{*}{\rotatebox{90}{\texttt{PKU}}} & \method & \textbf{2.78 $\pm$ 0.23} & \textbf{5.55 $\pm$ 0.28} & -- & \textbf{3.85 $\pm$ 0.31} & \textbf{7.68 $\pm$ 0.38} & -- \\
 & $f_0$ & 12.34 $\pm$ 0.40 & 22.47 $\pm$ 0.28 & -- & 18.08 $\pm$ 0.61 & 32.92 $\pm$ 0.47 & -- \\
 & \texttt{BPC} & 7.91 $\pm$ 0.25 & 19.25 $\pm$ 0.39 & -- & 11.22 $\pm$ 0.36 & 27.28 $\pm$ 0.56 & -- \\
 & \texttt{CalibraEval} & 7.66 $\pm$ 0.19 & 20.51 $\pm$ 0.35 & -- & 10.18 $\pm$ 0.26 & 27.26 $\pm$ 0.47 & -- \\
 & \texttt{PORTIA} & 10.26 $\pm$ 0.10 & 21.06 $\pm$ 0.31 & -- & 14.73 $\pm$ 0.19 & 30.22 $\pm$ 0.45 & -- \\
 & \texttt{LenControl} & 12.40 $\pm$ 0.54 & 21.16 $\pm$ 0.70 & -- & 18.16 $\pm$ 0.75 & 31.00 $\pm$ 1.02 & -- \\
\midrule
\multirow{6}{*}{\rotatebox{90}{\texttt{AEX}}} & \method & \textbf{3.72 $\pm$ 0.35} & \textbf{4.49 $\pm$ 0.27} & \textbf{2.58 $\pm$ 0.16} & \textbf{4.27 $\pm$ 0.41} & \textbf{5.14 $\pm$ 0.29} & \textbf{2.96 $\pm$ 0.19} \\
 & $f_0$ & 16.11 $\pm$ 0.33 & 16.16 $\pm$ 0.33 & 16.11 $\pm$ 0.33 & 20.82 $\pm$ 0.39 & 20.88 $\pm$ 0.40 & 20.82 $\pm$ 0.39 \\
 & \texttt{BPC} & 8.21 $\pm$ 0.36 & 10.46 $\pm$ 0.35 & 8.21 $\pm$ 0.36 & 10.09 $\pm$ 0.42 & 12.84 $\pm$ 0.41 & 10.09 $\pm$ 0.42 \\
 & \texttt{CalibraEval} & 11.18 $\pm$ 0.51 & 12.44 $\pm$ 0.46 & 11.18 $\pm$ 0.51 & 14.06 $\pm$ 0.66 & 15.63 $\pm$ 0.58 & 14.06 $\pm$ 0.66 \\
 & \texttt{PORTIA} & 10.05 $\pm$ 0.47 & 10.06 $\pm$ 0.47 & 10.05 $\pm$ 0.47 & 12.79 $\pm$ 0.58 & 12.80 $\pm$ 0.58 & 12.79 $\pm$ 0.58 \\
 & \texttt{LenControl} & 12.96 $\pm$ 0.34 & 13.15 $\pm$ 0.24 & 12.96 $\pm$ 0.34 & 16.86 $\pm$ 0.42 & 17.11 $\pm$ 0.30 & 16.86 $\pm$ 0.42 \\
\midrule
\multirow{6}{*}{{\rotatebox{90}{\texttt{PPE}}}} & \method & \textbf{5.41 $\pm$ 0.18} & \textbf{7.48 $\pm$ 0.41} & \textbf{4.77 $\pm$ 0.34} & \textbf{2.63 $\pm$ 0.10} & \textbf{3.62 $\pm$ 0.19} & \textbf{2.32 $\pm$ 0.17} \\
 & $f_0$ & 15.01 $\pm$ 0.65 & 16.08 $\pm$ 0.52 & 15.01 $\pm$ 0.65 & 8.35 $\pm$ 0.38 & 8.93 $\pm$ 0.30 & 8.35 $\pm$ 0.38 \\
 & \texttt{BPC} & 9.05 $\pm$ 0.60 & 12.36 $\pm$ 0.40 & 9.24 $\pm$ 0.51 & 4.80 $\pm$ 0.33 & 6.55 $\pm$ 0.21 & 4.91 $\pm$ 0.28 \\
 & \texttt{CalibraEval} & 9.25 $\pm$ 1.77 & 13.67 $\pm$ 0.71 & 9.06 $\pm$ 1.71 & 4.91 $\pm$ 1.07 & 7.12 $\pm$ 0.59 & 4.81 $\pm$ 1.05 \\
 & \texttt{PORTIA} & 10.15 $\pm$ 0.42 & 10.71 $\pm$ 0.32 & 10.15 $\pm$ 0.42 & 5.55 $\pm$ 0.23 & 5.86 $\pm$ 0.18 & 5.55 $\pm$ 0.23 \\
 & \texttt{LenControl} & 13.82 $\pm$ 0.58 & 14.47 $\pm$ 0.18 & 13.82 $\pm$ 0.58 & 7.75 $\pm$ 0.34 & 8.10 $\pm$ 0.12 & 7.75 $\pm$ 0.34 \\
\midrule
\multirow{6}{*}{\rotatebox{90}{\texttt{RWB}}} & \method & 6.90 $\pm$ 1.11 & \textbf{5.22 $\pm$ 0.41} & 5.21 $\pm$ 0.27 & 3.70 $\pm$ 0.57 & \textbf{2.82 $\pm$ 0.25} & 2.81 $\pm$ 0.18 \\
 & $f_0$ & 7.35 $\pm$ 0.91 & 13.08 $\pm$ 0.42 & 6.78 $\pm$ 0.67 & 3.35 $\pm$ 0.41 & 5.98 $\pm$ 0.15 & 3.09 $\pm$ 0.28 \\
 & \texttt{BPC} & \textbf{6.10 $\pm$ 0.26} & 10.92 $\pm$ 0.53 & 5.69 $\pm$ 0.26 & \textbf{2.48 $\pm$ 0.09} & 4.45 $\pm$ 0.22 & 2.32 $\pm$ 0.11 \\
 & \texttt{CalibraEval} & 7.70 $\pm$ 0.43 & 10.71 $\pm$ 0.92 & 7.29 $\pm$ 0.22 & 3.14 $\pm$ 0.18 & 4.42 $\pm$ 0.55 & 2.99 $\pm$ 0.20 \\
 & \texttt{PORTIA} & 6.24 $\pm$ 0.73 & 10.21 $\pm$ 0.60 & 4.83 $\pm$ 0.38 & 2.80 $\pm$ 0.31 & 4.59 $\pm$ 0.23 & \textbf{2.17 $\pm$ 0.15} \\
 & \texttt{LenControl} & 9.55 $\pm$ 0.73 & 7.62 $\pm$ 0.50 & \textbf{4.68 $\pm$ 0.40} & 4.94 $\pm$ 0.35 & 3.95 $\pm$ 0.28 & 2.42 $\pm$ 0.21 \\
\midrule
\multirow{6}{*}{\rotatebox{90}{\texttt{MTB}}} & \method & \textbf{13.42 $\pm$ 1.59} & \textbf{13.46 $\pm$ 1.28} & 11.89 $\pm$ 0.79 & \textbf{2.23 $\pm$ 0.26} & \textbf{2.20 $\pm$ 0.20} & 1.96 $\pm$ 0.12 \\
 & $f_0$ & 15.86 $\pm$ 2.03 & 16.15 $\pm$ 0.82 & 12.00 $\pm$ 0.92 & 2.86 $\pm$ 0.40 & 2.85 $\pm$ 0.09 & 2.15 $\pm$ 0.22 \\
 & \texttt{BPC} & 14.30 $\pm$ 2.05 & 15.76 $\pm$ 1.11 & \textbf{10.51 $\pm$ 1.04} & 2.43 $\pm$ 0.37 & 2.63 $\pm$ 0.16 & \textbf{1.79 $\pm$ 0.19} \\
 & \texttt{CalibraEval} & 16.67 $\pm$ 2.23 & 17.92 $\pm$ 1.12 & 13.47 $\pm$ 1.88 & 2.91 $\pm$ 0.42 & 3.09 $\pm$ 0.20 & 2.35 $\pm$ 0.36 \\
 & \texttt{PORTIA} & 14.50 $\pm$ 1.27 & 21.07 $\pm$ 2.76 & 15.12 $\pm$ 2.94 & 2.42 $\pm$ 0.24 & 3.44 $\pm$ 0.39 & 2.56 $\pm$ 0.55 \\
 & \texttt{LenControl} & 14.12 $\pm$ 1.99 & 14.63 $\pm$ 0.65 & 13.54 $\pm$ 1.74 & 2.63 $\pm$ 0.38 & 2.69 $\pm$ 0.10 & 2.53 $\pm$ 0.34 \\
\bottomrule
\end{tabular}}
\label{tab:APP_q1_MCE_perc_and_sigma}
\end{table}

\begin{table}[]
    \centering
    \setlength{\tabcolsep}{4pt}
\caption{Average per-group $\text{MCE}_\%$ and $\text{MCE}_{\sigma}$ ($\pm$ standard error) for \method and \methodG. The \emph{family} axis does not show \texttt{PKU} since it does not contain \texttt{muse}-family responses.}
    \scalebox{0.8}{\begin{tabular}{ll|ccc|ccc}
\toprule
 &  & \multicolumn{3}{c|}{$\text{MCE}_{\%}$} & \multicolumn{3}{c}{$\text{MCE}_{\sigma}$} \\
\cmidrule(lr){3-5}\cmidrule(lr){6-8}
Dataset & Method & Position & Verbosity & Family & Position & Verbosity & Family \\
\midrule
\multirow{2}{*}{PKU} & \method & 2.78 $\pm$ 0.23 & 5.55 $\pm$ 0.28 & -- & 3.85 $\pm$ 0.31 & 7.68 $\pm$ 0.38 & -- \\
 & \methodG & \textbf{2.56 $\pm$ 0.10} & \textbf{2.81 $\pm$ 0.18} & -- & \textbf{3.58 $\pm$ 0.14} & \textbf{3.92 $\pm$ 0.25} & -- \\
\midrule
\multirow{2}{*}{AEX} & \method & 3.72 $\pm$ 0.35 & 4.49 $\pm$ 0.27 & 2.58 $\pm$ 0.16 & 4.27 $\pm$ 0.41 & 5.14 $\pm$ 0.29 & 2.96 $\pm$ 0.19 \\
 & \methodG & \textbf{2.52 $\pm$ 0.30} & \textbf{2.55 $\pm$ 0.24} & \textbf{2.50 $\pm$ 0.19} & \textbf{2.90 $\pm$ 0.34} & \textbf{2.93 $\pm$ 0.29} & \textbf{2.87 $\pm$ 0.22} \\
\midrule
\multirow{2}{*}{PPE} & \method & 5.41 $\pm$ 0.18 & 7.48 $\pm$ 0.41 & 4.77 $\pm$ 0.34 & 2.63 $\pm$ 0.10 & 3.62 $\pm$ 0.19 & 2.32 $\pm$ 0.17 \\
 & \methodG & \textbf{3.67 $\pm$ 0.30} & \textbf{3.99 $\pm$ 0.18} & \textbf{4.56 $\pm$ 0.41} & \textbf{1.79 $\pm$ 0.15} & \textbf{1.95 $\pm$ 0.09} & \textbf{2.23 $\pm$ 0.21} \\
\midrule
\multirow{2}{*}{RWB} & \method & 6.90 $\pm$ 1.11 & 5.22 $\pm$ 0.41 & 5.21 $\pm$ 0.27 & 3.70 $\pm$ 0.57 & \textbf{2.82 $\pm$ 0.25} & \textbf{2.81 $\pm$ 0.18} \\
 & \methodG & \textbf{4.66 $\pm$ 0.45} & \textbf{4.70 $\pm$ 0.40} & \textbf{4.64 $\pm$ 0.15} & \textbf{2.83 $\pm$ 0.21} & 2.87 $\pm$ 0.22 & 2.83 $\pm$ 0.10 \\
\midrule
\multirow{2}{*}{MTB} & \method & 13.42 $\pm$ 1.59 & 13.46 $\pm$ 1.28 & 11.89 $\pm$ 0.79 & 2.23 $\pm$ 0.26 & 2.20 $\pm$ 0.20 & 1.96 $\pm$ 0.12 \\
 & \methodG & \textbf{12.53 $\pm$ 1.54} & \textbf{12.33 $\pm$ 0.86} & \textbf{11.63 $\pm$ 0.23} & \textbf{2.10 $\pm$ 0.24} & \textbf{2.05 $\pm$ 0.14} & \textbf{1.91 $\pm$ 0.04} \\
\bottomrule
\end{tabular}}
\label{tab:APP_q2_MCE_perc_and_sigma}
\end{table}

\begin{table}[]
    \centering
        \caption{Average per-group $\text{MCE}_\%$ and $\text{MCE}_{\sigma}$ ($\pm$ standard error) for \method and the mitigation baselines across datasets over the three biases documented in the literature (\ie position, verbosity, and family) and the \emph{unspecified} groups.}
    \begin{tabular}{ll|cc|cc}
\toprule
 &  & \multicolumn{2}{c|}{$\text{MCE}_{\%}$} & \multicolumn{2}{c}{$\texttt{MCE}_\sigma$} \\
\cmidrule(lr){3-4}\cmidrule(lr){5-6}
Dataset & Method & Specified & Unspecified & Specified & Unspecified \\
\midrule
\multirow{6}{*}{\rotatebox{90}{\texttt{PKU}}} & \method & \textbf{5.55 $\pm$ 0.28} & \textbf{3.89 $\pm$ 0.36} & \textbf{7.68 $\pm$ 0.38} & \textbf{5.38 $\pm$ 0.50} \\
 & $f_0$ & 22.46 $\pm$ 0.28 & 12.34 $\pm$ 0.40 & 32.91 $\pm$ 0.47 & 18.08 $\pm$ 0.61 \\
 & \texttt{BPC} & 19.23 $\pm$ 0.40 & 7.91 $\pm$ 0.25 & 27.26 $\pm$ 0.56 & 11.22 $\pm$ 0.36 \\
 & \texttt{CalibraEval} & 20.49 $\pm$ 0.35 & 7.66 $\pm$ 0.19 & 27.24 $\pm$ 0.47 & 10.18 $\pm$ 0.26 \\
 & \texttt{PORTIA} & 21.04 $\pm$ 0.31 & 10.26 $\pm$ 0.10 & 30.20 $\pm$ 0.45 & 14.73 $\pm$ 0.19 \\
 & \texttt{LenControl} & 21.15 $\pm$ 0.71 & 12.40 $\pm$ 0.54 & 30.99 $\pm$ 1.02 & 18.16 $\pm$ 0.75 \\
\midrule
\multirow{6}{*}{\rotatebox{90}{\texttt{AEX}}} & \method & \textbf{5.09 $\pm$ 0.19} & \textbf{3.29 $\pm$ 0.14} & \textbf{5.83 $\pm$ 0.21} & \textbf{3.77 $\pm$ 0.17} \\
 & $f_0$ & 16.17 $\pm$ 0.33 & 16.11 $\pm$ 0.33 & 20.89 $\pm$ 0.40 & 20.82 $\pm$ 0.39 \\
 & \texttt{BPC} & 10.46 $\pm$ 0.35 & 8.21 $\pm$ 0.36 & 12.84 $\pm$ 0.41 & 10.09 $\pm$ 0.42 \\
 & \texttt{CalibraEval} & 12.44 $\pm$ 0.46 & 11.18 $\pm$ 0.51 & 15.63 $\pm$ 0.58 & 14.06 $\pm$ 0.66 \\
 & \texttt{PORTIA} & 10.06 $\pm$ 0.47 & 10.05 $\pm$ 0.47 & 12.80 $\pm$ 0.58 & 12.79 $\pm$ 0.58 \\
 & \texttt{LenControl} & 13.15 $\pm$ 0.24 & 12.96 $\pm$ 0.34 & 17.11 $\pm$ 0.30 & 16.86 $\pm$ 0.42 \\
\midrule
\multirow{6}{*}{{\rotatebox{90}{\texttt{PPE}}}} & \method & \textbf{7.97 $\pm$ 0.43} & \textbf{7.16 $\pm$ 0.21} & \textbf{3.87 $\pm$ 0.20} & \textbf{3.47 $\pm$ 0.10} \\
 & $f_0$ & 16.05 $\pm$ 0.56 & 15.01 $\pm$ 0.65 & 8.93 $\pm$ 0.32 & 8.35 $\pm$ 0.38 \\
 & \texttt{BPC} & 12.36 $\pm$ 0.41 & 11.31 $\pm$ 0.55 & 6.56 $\pm$ 0.21 & 6.00 $\pm$ 0.29 \\
 & \texttt{CalibraEval} & 13.64 $\pm$ 0.78 & 12.69 $\pm$ 0.76 & 7.12 $\pm$ 0.63 & 6.61 $\pm$ 0.56 \\
 & \texttt{PORTIA} & 10.77 $\pm$ 0.37 & 10.66 $\pm$ 0.62 & 5.90 $\pm$ 0.21 & 5.83 $\pm$ 0.33 \\
 & \texttt{LenControl} & 14.47 $\pm$ 0.18 & 14.06 $\pm$ 0.36 & 8.11 $\pm$ 0.12 & 7.89 $\pm$ 0.22 \\
\midrule
\multirow{6}{*}{\rotatebox{90}{\texttt{RWB}}} & \method & \textbf{7.27 $\pm$ 0.89} & \textbf{8.00 $\pm$ 0.43} & \textbf{3.90 $\pm$ 0.45} & \textbf{4.29 $\pm$ 0.16} \\
 & $f_0$ & 13.08 $\pm$ 0.42 & 13.77 $\pm$ 1.12 & 5.98 $\pm$ 0.15 & 6.28 $\pm$ 0.46 \\
 & \texttt{BPC} & 11.71 $\pm$ 0.27 & 10.90 $\pm$ 0.49 & 4.77 $\pm$ 0.08 & 4.44 $\pm$ 0.18 \\
 & \texttt{CalibraEval} & 11.05 $\pm$ 0.85 & 13.42 $\pm$ 0.22 & 4.54 $\pm$ 0.50 & 5.52 $\pm$ 0.41 \\
 & \texttt{PORTIA} & 10.83 $\pm$ 0.67 & 11.23 $\pm$ 0.44 & 4.87 $\pm$ 0.27 & 5.06 $\pm$ 0.25 \\
 & \texttt{LenControl} & 10.82 $\pm$ 0.67 & 8.87 $\pm$ 0.46 & 5.60 $\pm$ 0.35 & 4.58 $\pm$ 0.17 \\
\midrule
\multirow{6}{*}{\rotatebox{90}{\texttt{MTB}}} & \method & \textbf{16.36 $\pm$ 0.84} & \textbf{18.80 $\pm$ 0.57} & \textbf{2.70 $\pm$ 0.11} & \textbf{3.12 $\pm$ 0.13} \\
 & $f_0$ & 19.59 $\pm$ 0.56 & 21.74 $\pm$ 0.45 & 3.50 $\pm$ 0.12 & 3.88 $\pm$ 0.05 \\
 & \texttt{BPC} & 18.00 $\pm$ 1.11 & 19.00 $\pm$ 0.69 & 3.04 $\pm$ 0.19 & 3.21 $\pm$ 0.12 \\
 & \texttt{CalibraEval} & 20.40 $\pm$ 0.96 & 23.38 $\pm$ 1.13 & 3.55 $\pm$ 0.20 & 4.07 $\pm$ 0.22 \\
 & \texttt{PORTIA} & 21.21 $\pm$ 2.32 & 26.31 $\pm$ 3.89 & 3.51 $\pm$ 0.35 & 4.37 $\pm$ 0.63 \\
 & \texttt{LenControl} & 17.26 $\pm$ 1.05 & 21.16 $\pm$ 0.72 & 3.20 $\pm$ 0.19 & 3.93 $\pm$ 0.14 \\
\bottomrule
\end{tabular}
\label{tab:APP_specified_unspecified}
\end{table}

\begin{table}[]
    \centering
   \caption{Average per-group $\text{MCE}_\%$  and $\text{MCE}_{\sigma}$($\pm$ standard error) for \method and the mitigation baselines across datasets with \texttt{Qwen~2.5-7B-Instruct} used as the judge. The \emph{family} axis contains \texttt{AEX} and \texttt{PPE} only: \texttt{PKU}, \texttt{RWB} and \texttt{MTB} do not contain \texttt{qwen}-family responses.}
    \scalebox{0.8}{\begin{tabular}{ll|ccc|ccc}
\toprule
 &  & \multicolumn{3}{c|}{$\text{MCE}_{\%}$} & \multicolumn{3}{c}{$\text{MCE}_{\sigma}$} \\
\cmidrule(lr){3-5}\cmidrule(lr){6-8}
Dataset & Method & Position & Verbosity & Family & Position & Verbosity & Family \\
\midrule
\multirow{6}{*}{\rotatebox{90}{\texttt{PKU}}} & \method & \textbf{2.39 $\pm$ 0.17} & \textbf{6.32 $\pm$ 0.53} & -- & \textbf{3.17 $\pm$ 0.23} & \textbf{8.39 $\pm$ 0.69} & -- \\
 & $f_0$ & 19.81 $\pm$ 0.28 & 19.80 $\pm$ 0.28 & -- & 31.48 $\pm$ 0.45 & 31.45 $\pm$ 0.45 & -- \\
 & \texttt{BPC} & 8.86 $\pm$ 0.15 & 9.43 $\pm$ 0.13 & -- & 12.91 $\pm$ 0.21 & 13.73 $\pm$ 0.20 & -- \\
 & \texttt{CalibraEval} & 8.13 $\pm$ 0.33 & 11.05 $\pm$ 0.38 & -- & 10.95 $\pm$ 0.48 & 14.87 $\pm$ 0.49 & -- \\
 & \texttt{PORTIA} & 10.47 $\pm$ 0.10 & 17.09 $\pm$ 0.14 & -- & 14.50 $\pm$ 0.12 & 23.66 $\pm$ 0.21 & -- \\
 & \texttt{LenControl} & 16.56 $\pm$ 0.11 & 16.55 $\pm$ 0.11 & -- & 26.13 $\pm$ 0.18 & 26.10 $\pm$ 0.18 & -- \\
\midrule
\multirow{6}{*}{\rotatebox{90}{\texttt{AEX}}} & \method & \textbf{6.78 $\pm$ 0.39} & \textbf{4.63 $\pm$ 0.29} & \textbf{2.70 $\pm$ 0.27} & \textbf{7.37 $\pm$ 0.44} & \textbf{5.03 $\pm$ 0.31} & \textbf{2.93 $\pm$ 0.29} \\
 & $f_0$ & 40.95 $\pm$ 0.42 & 40.95 $\pm$ 0.42 & 40.95 $\pm$ 0.42 & 58.37 $\pm$ 0.55 & 58.35 $\pm$ 0.55 & 58.37 $\pm$ 0.55 \\
 & \texttt{BPC} & 10.44 $\pm$ 0.30 & 9.92 $\pm$ 0.27 & 9.92 $\pm$ 0.27 & 12.39 $\pm$ 0.36 & 11.77 $\pm$ 0.31 & 11.77 $\pm$ 0.30 \\
 & \texttt{CalibraEval} & 7.61 $\pm$ 0.28 & 9.12 $\pm$ 0.38 & 5.40 $\pm$ 0.10 & 8.12 $\pm$ 0.29 & 9.72 $\pm$ 0.39 & 5.76 $\pm$ 0.09 \\
 & \texttt{PORTIA} & 13.79 $\pm$ 0.26 & 13.75 $\pm$ 0.29 & 13.79 $\pm$ 0.26 & 17.11 $\pm$ 0.32 & 17.06 $\pm$ 0.35 & 17.11 $\pm$ 0.32 \\
 & \texttt{LenControl} & 16.79 $\pm$ 0.30 & 16.79 $\pm$ 0.30 & 16.79 $\pm$ 0.30 & 21.45 $\pm$ 0.37 & 21.45 $\pm$ 0.37 & 21.45 $\pm$ 0.37 \\
\midrule
\multirow{6}{*}{\rotatebox{90}{\texttt{PPE}}} & \method & \textbf{8.27 $\pm$ 0.45} & \textbf{7.19 $\pm$ 0.41} & \textbf{5.74 $\pm$ 0.91} & \textbf{3.81 $\pm$ 0.20} & \textbf{3.31 $\pm$ 0.19} & \textbf{2.64 $\pm$ 0.42} \\
 & $f_0$ & 39.11 $\pm$ 0.72 & 39.18 $\pm$ 0.71 & 39.11 $\pm$ 0.72 & 24.17 $\pm$ 0.41 & 24.17 $\pm$ 0.39 & 24.17 $\pm$ 0.41 \\
 & \texttt{BPC} & 11.98 $\pm$ 0.57 & 11.49 $\pm$ 0.68 & 11.46 $\pm$ 0.72 & 6.21 $\pm$ 0.31 & 5.94 $\pm$ 0.37 & 5.94 $\pm$ 0.39 \\
 & \texttt{CalibraEval} & 13.01 $\pm$ 2.65 & 12.21 $\pm$ 2.54 & 11.02 $\pm$ 2.96 & 6.89 $\pm$ 1.97 & 6.46 $\pm$ 1.87 & 5.91 $\pm$ 2.07 \\
 & \texttt{PORTIA} & 14.63 $\pm$ 0.59 & 14.57 $\pm$ 0.57 & 14.55 $\pm$ 0.58 & 7.84 $\pm$ 0.29 & 7.79 $\pm$ 0.28 & 7.80 $\pm$ 0.29 \\
 & \texttt{LenControl} & 19.47 $\pm$ 0.69 & 19.43 $\pm$ 0.68 & 19.47 $\pm$ 0.69 & 10.97 $\pm$ 0.38 & 10.93 $\pm$ 0.38 & 10.97 $\pm$ 0.38 \\
\midrule
\multirow{6}{*}{\rotatebox{90}{\texttt{RWB}}} & \method & 9.98 $\pm$ 1.21 & \textbf{9.70 $\pm$ 1.04} & -- & 3.51 $\pm$ 0.44 & \textbf{3.42 $\pm$ 0.38} & -- \\
 & $f_0$ & 13.66 $\pm$ 1.07 & 22.11 $\pm$ 1.42 & -- & 5.24 $\pm$ 0.41 & 8.47 $\pm$ 0.54 & -- \\
 & \texttt{BPC} & \textbf{8.64 $\pm$ 1.15} & 16.87 $\pm$ 1.21 & -- & \textbf{2.96 $\pm$ 0.42} & 5.74 $\pm$ 0.40 & -- \\
 & \texttt{CalibraEval} & 11.04 $\pm$ 1.24 & 17.03 $\pm$ 1.76 & -- & 3.83 $\pm$ 0.60 & 5.89 $\pm$ 0.85 & -- \\
 & \texttt{PORTIA} & 12.61 $\pm$ 1.31 & 16.30 $\pm$ 1.84 & -- & 4.00 $\pm$ 0.42 & 5.17 $\pm$ 0.57 & -- \\
 & \texttt{LenControl} & 14.78 $\pm$ 1.22 & 12.72 $\pm$ 1.16 & -- & 5.99 $\pm$ 0.47 & 5.15 $\pm$ 0.44 & -- \\
\midrule
\multirow{6}{*}{\rotatebox{90}{\texttt{MTB}}} & \method & \textbf{11.83 $\pm$ 0.77} & 14.22 $\pm$ 0.69 & -- & \textbf{1.85 $\pm$ 0.16} & \textbf{2.19 $\pm$ 0.13} & -- \\
 & $f_0$ & 13.88 $\pm$ 1.07 & 14.25 $\pm$ 1.03 & -- & 2.39 $\pm$ 0.23 & 2.44 $\pm$ 0.23 & -- \\
 & \texttt{BPC} & 14.28 $\pm$ 1.10 & 14.22 $\pm$ 1.23 & -- & 2.23 $\pm$ 0.16 & 2.20 $\pm$ 0.15 & -- \\
 & \texttt{CalibraEval} & 13.80 $\pm$ 1.33 & 20.30 $\pm$ 1.83 & -- & 2.16 $\pm$ 0.29 & 3.06 $\pm$ 0.24 & -- \\
 & \texttt{PORTIA} & 16.95 $\pm$ 2.67 & 24.30 $\pm$ 3.90 & -- & 2.48 $\pm$ 0.39 & 3.52 $\pm$ 0.54 & -- \\
 & \texttt{LenControl} & 15.93 $\pm$ 1.08 & \textbf{14.18 $\pm$ 0.86} & -- & 2.94 $\pm$ 0.25 & 2.58 $\pm$ 0.16 & -- \\
\bottomrule
\end{tabular}}
\label{tab:APP_q1_qwenjudge_MCE}
\end{table}

\begin{table}[]
    \centering
    \caption{Average per-group $\text{MCE}_\%$ and $\text{MCE}_{\sigma}$  ($\pm 1$ s.e.) for \method (embeddings from the judge's last hidden state), \method{\tt (Qwen3-0.6B)} (embeddings from \texttt{Qwen3-Embedding-0.6B}) and, as a reference, the \emph{best} of the \cref{fig:Q1} baselines across datasets. Family is undefined on \texttt{PKU}, since it contains no response from the judge's own family.}
\setlength{\tabcolsep}{4pt}
\scalebox{0.75}{\begin{tabular}{ll|ccc|ccc}
\toprule
 &  & \multicolumn{3}{c|}{$\text{MCE}_{\%}$} & \multicolumn{3}{c}{$\text{MCE}_{\sigma}$} \\
\cmidrule(lr){3-5}\cmidrule(lr){6-8}
Dataset & Method & Position & Verbosity & Family & Position & Verbosity & Family \\
\midrule
\multirow{3}{*}{\rotatebox{90}{\texttt{PKU}}} & \method & 2.78 $\pm$ 0.23 & \textbf{5.55 $\pm$ 0.28} & -- & 3.85 $\pm$ 0.31 & \textbf{7.68 $\pm$ 0.38} & -- \\
 & \method{} (Qwen3-0.6B) & \textbf{2.73 $\pm$ 0.22} & 7.39 $\pm$ 0.62 & -- & \textbf{3.74 $\pm$ 0.31} & 10.12 $\pm$ 0.84 & -- \\
 & \emph{best baseline} & 7.66 $\pm$ 0.19 & 19.25 $\pm$ 0.39 & -- & 10.18 $\pm$ 0.26 & 27.26 $\pm$ 0.47 & -- \\
\midrule
\multirow{3}{*}{\rotatebox{90}{\texttt{AEX}}} & \method & \textbf{3.72 $\pm$ 0.35} & \textbf{4.49 $\pm$ 0.27} & 2.58 $\pm$ 0.16 & \textbf{4.27 $\pm$ 0.41} & \textbf{5.14 $\pm$ 0.29} & 2.96 $\pm$ 0.19 \\
 & \method{} (Qwen3-0.6B) & 4.17 $\pm$ 0.31 & 5.49 $\pm$ 0.28 & \textbf{2.28 $\pm$ 0.14} & 4.70 $\pm$ 0.36 & 6.19 $\pm$ 0.31 & \textbf{2.57 $\pm$ 0.16} \\
 & \emph{best baseline} & 8.21 $\pm$ 0.36 & 10.06 $\pm$ 0.47 & 8.21 $\pm$ 0.36 & 10.09 $\pm$ 0.42 & 12.80 $\pm$ 0.58 & 10.09 $\pm$ 0.42 \\
\midrule
\multirow{3}{*}{\rotatebox{90}{\texttt{PPE}}} & \method & \textbf{5.41 $\pm$ 0.18} & \textbf{7.48 $\pm$ 0.41} & \textbf{4.77 $\pm$ 0.34} & \textbf{2.63 $\pm$ 0.10} & \textbf{3.62 $\pm$ 0.19} & \textbf{2.32 $\pm$ 0.17} \\
 & \method{} (Qwen3-0.6B) & 5.65 $\pm$ 0.27 & 8.60 $\pm$ 0.55 & 5.08 $\pm$ 0.27 & 2.73 $\pm$ 0.14 & 4.14 $\pm$ 0.25 & 2.45 $\pm$ 0.14 \\
 & \emph{best baseline} & 9.05 $\pm$ 0.60 & 10.71 $\pm$ 0.32 & 9.06 $\pm$ 1.71 & 4.80 $\pm$ 0.33 & 5.86 $\pm$ 0.18 & 4.81 $\pm$ 1.05 \\
\midrule
\multirow{3}{*}{\rotatebox{90}{\texttt{RWB}}} & \method & 6.90 $\pm$ 1.11 & \textbf{5.22 $\pm$ 0.41} & 5.21 $\pm$ 0.27 & 3.70 $\pm$ 0.57 & \textbf{2.82 $\pm$ 0.25} & 2.81 $\pm$ 0.18 \\
 & \method{} (Qwen3-0.6B) & 6.67 $\pm$ 0.95 & 7.88 $\pm$ 0.69 & 6.02 $\pm$ 0.63 & 2.94 $\pm$ 0.43 & 3.44 $\pm$ 0.22 & 2.64 $\pm$ 0.27 \\
 & \emph{best baseline} & \textbf{6.10 $\pm$ 0.26} & 7.62 $\pm$ 0.50 & \textbf{4.68 $\pm$ 0.40} & \textbf{2.48 $\pm$ 0.09} & 3.95 $\pm$ 0.28 & \textbf{2.17 $\pm$ 0.15} \\
\midrule
\multirow{3}{*}{\rotatebox{90}{\texttt{MTB}}} & \method & \textbf{13.42 $\pm$ 1.59} & \textbf{13.46 $\pm$ 1.28} & 11.89 $\pm$ 0.79 & \textbf{2.23 $\pm$ 0.26} & \textbf{2.20 $\pm$ 0.20} & 1.96 $\pm$ 0.12 \\
 & \method{} (Qwen3-0.6B) & 14.15 $\pm$ 2.50 & 14.31 $\pm$ 1.68 & \textbf{10.43 $\pm$ 0.76} & 2.44 $\pm$ 0.47 & 2.39 $\pm$ 0.27 & \textbf{1.74 $\pm$ 0.16} \\
 & \emph{best baseline} & 14.12 $\pm$ 1.99 & 14.63 $\pm$ 0.65 & 10.51 $\pm$ 1.04 & 2.42 $\pm$ 0.24 & 2.63 $\pm$ 0.16 & 1.79 $\pm$ 0.19 \\
\bottomrule
\end{tabular}  } 
\label{tab:APP_q1_qwenjudge_MCE_embed}
\end{table}

\begin{table}[]
    \centering
\setlength{\tabcolsep}{4pt}
\caption{Average per-group $\text{MCE}_\%$ and  $\text{MCE}_{\sigma}$ ($\pm 1$ s.e.) for \method and the calibration and multicalibration baselines across datasets. Family is undefined on \texttt{PKU}, since it contains no response from the judge's own family.}
    \scalebox{0.8}{\begin{tabular}{ll|ccc|ccc}
\toprule
 &  & \multicolumn{3}{c|}{$\text{MCE}_{\%}$} & \multicolumn{3}{c}{$\text{MCE}_{\sigma}$} \\
\cmidrule(lr){3-5}\cmidrule(lr){6-8}
Dataset & Method & Position & Verbosity & Family & Position & Verbosity & Family \\
\midrule
\multirow{6}{*}{\rotatebox{90}{\texttt{PKU}}} & \method & \textbf{2.78 $\pm$ 0.23} & \textbf{5.55 $\pm$ 0.28} & -- & 3.85 $\pm$ 0.31 & \textbf{7.68 $\pm$ 0.38} & -- \\
 & $f_0$ & 12.34 $\pm$ 0.40 & 22.47 $\pm$ 0.28 & -- & 18.08 $\pm$ 0.61 & 32.92 $\pm$ 0.47 & -- \\
 & \texttt{Temperature} & 5.56 $\pm$ 0.28 & 19.84 $\pm$ 0.20 & -- & 7.10 $\pm$ 0.36 & 25.32 $\pm$ 0.29 & -- \\
 & \texttt{Platt} & 5.29 $\pm$ 0.20 & 18.64 $\pm$ 0.51 & -- & 6.76 $\pm$ 0.26 & 23.80 $\pm$ 0.64 & -- \\
 & \texttt{LOGR} & 2.87 $\pm$ 0.15 & 17.71 $\pm$ 0.48 & -- & \textbf{3.68 $\pm$ 0.19} & 22.73 $\pm$ 0.61 & -- \\
 & \texttt{IGLB} & 3.33 $\pm$ 0.36 & 16.86 $\pm$ 0.47 & -- & 4.30 $\pm$ 0.46 & 21.78 $\pm$ 0.60 & -- \\
\midrule
\multirow{6}{*}{\rotatebox{90}{\texttt{AEX}}} & \method & \textbf{3.72 $\pm$ 0.35} & \textbf{4.49 $\pm$ 0.27} & 2.58 $\pm$ 0.16 & \textbf{4.27 $\pm$ 0.41} & \textbf{5.14 $\pm$ 0.29} & 2.96 $\pm$ 0.19 \\
 & $f_0$ & 16.11 $\pm$ 0.33 & 16.16 $\pm$ 0.33 & 16.11 $\pm$ 0.33 & 20.82 $\pm$ 0.39 & 20.88 $\pm$ 0.40 & 20.82 $\pm$ 0.39 \\
 & \texttt{Temperature} & 6.58 $\pm$ 0.22 & 11.52 $\pm$ 0.43 & 4.70 $\pm$ 0.22 & 7.11 $\pm$ 0.25 & 12.46 $\pm$ 0.46 & 5.08 $\pm$ 0.23 \\
 & \texttt{Platt} & 5.23 $\pm$ 0.51 & 9.50 $\pm$ 0.25 & 3.93 $\pm$ 0.22 & 5.67 $\pm$ 0.56 & 10.28 $\pm$ 0.26 & 4.26 $\pm$ 0.25 \\
 & \texttt{LOGR} & 4.13 $\pm$ 0.44 & 8.86 $\pm$ 0.27 & \textbf{1.99 $\pm$ 0.10} & 4.50 $\pm$ 0.48 & 9.64 $\pm$ 0.28 & \textbf{2.17 $\pm$ 0.10} \\
 & \texttt{IGLB} & 4.00 $\pm$ 0.35 & 8.73 $\pm$ 0.30 & 2.32 $\pm$ 0.07 & 4.38 $\pm$ 0.39 & 9.55 $\pm$ 0.31 & 2.53 $\pm$ 0.07 \\
\midrule
\multirow{6}{*}{\rotatebox{90}{\texttt{PPE}}} & \method & 5.41 $\pm$ 0.18 & \textbf{7.48 $\pm$ 0.41} & \textbf{4.77 $\pm$ 0.34} & 2.63 $\pm$ 0.10 & \textbf{3.62 $\pm$ 0.19} & \textbf{2.32 $\pm$ 0.17} \\
 & $f_0$ & 15.01 $\pm$ 0.65 & 16.08 $\pm$ 0.52 & 15.01 $\pm$ 0.65 & 8.35 $\pm$ 0.38 & 8.93 $\pm$ 0.30 & 8.35 $\pm$ 0.38 \\
 & \texttt{Temperature} & 6.37 $\pm$ 0.32 & 12.07 $\pm$ 0.87 & 5.12 $\pm$ 0.47 & 2.99 $\pm$ 0.15 & 5.65 $\pm$ 0.41 & 2.40 $\pm$ 0.22 \\
 & \texttt{Platt} & 5.75 $\pm$ 0.28 & 11.69 $\pm$ 0.47 & 5.11 $\pm$ 0.23 & 2.70 $\pm$ 0.13 & 5.48 $\pm$ 0.21 & 2.40 $\pm$ 0.11 \\
 & \texttt{LOGR} & \textbf{5.28 $\pm$ 0.40} & 11.33 $\pm$ 0.43 & 5.03 $\pm$ 0.24 & \textbf{2.49 $\pm$ 0.19} & 5.34 $\pm$ 0.19 & 2.37 $\pm$ 0.12 \\
 & \texttt{IGLB} & 5.74 $\pm$ 0.26 & 11.59 $\pm$ 0.42 & 5.51 $\pm$ 0.35 & 2.73 $\pm$ 0.14 & 5.51 $\pm$ 0.18 & 2.62 $\pm$ 0.16 \\
\midrule
\multirow{6}{*}{\rotatebox{90}{\texttt{RWB}}} & \method & 6.90 $\pm$ 1.11 & \textbf{5.22 $\pm$ 0.41} & \textbf{5.21 $\pm$ 0.27} & 3.70 $\pm$ 0.57 & \textbf{2.82 $\pm$ 0.25} & 2.81 $\pm$ 0.18 \\
 & $f_0$ & 7.35 $\pm$ 0.91 & 13.08 $\pm$ 0.42 & 6.78 $\pm$ 0.67 & 3.35 $\pm$ 0.41 & 5.98 $\pm$ 0.15 & 3.09 $\pm$ 0.28 \\
 & \texttt{Temperature} & 8.38 $\pm$ 1.07 & 12.16 $\pm$ 0.38 & 6.21 $\pm$ 0.68 & 3.16 $\pm$ 0.41 & 4.58 $\pm$ 0.13 & 2.34 $\pm$ 0.25 \\
 & \texttt{Platt} & 7.17 $\pm$ 0.76 & 10.33 $\pm$ 0.29 & 6.24 $\pm$ 0.48 & 2.71 $\pm$ 0.29 & 3.90 $\pm$ 0.09 & 2.36 $\pm$ 0.18 \\
 & \texttt{LOGR} & 7.65 $\pm$ 0.74 & 11.88 $\pm$ 0.31 & 6.43 $\pm$ 0.73 & 2.92 $\pm$ 0.28 & 4.53 $\pm$ 0.10 & 2.45 $\pm$ 0.27 \\
 & \texttt{IGLB} & \textbf{6.50 $\pm$ 0.68} & 11.20 $\pm$ 0.76 & 5.89 $\pm$ 0.71 & \textbf{2.45 $\pm$ 0.23} & 4.23 $\pm$ 0.27 & \textbf{2.22 $\pm$ 0.23} \\
\midrule
\multirow{6}{*}{\rotatebox{90}{\texttt{MTB}}} & \method & \textbf{13.42 $\pm$ 1.59} & \textbf{13.46 $\pm$ 1.28} & \textbf{11.89 $\pm$ 0.79} & \textbf{2.23 $\pm$ 0.26} & \textbf{2.20 $\pm$ 0.20} & \textbf{1.96 $\pm$ 0.12} \\
 & $f_0$ & 15.86 $\pm$ 2.03 & 16.15 $\pm$ 0.82 & 12.00 $\pm$ 0.92 & 2.86 $\pm$ 0.40 & 2.85 $\pm$ 0.09 & 2.15 $\pm$ 0.22 \\
 & \texttt{Temperature} & 15.37 $\pm$ 1.99 & 18.31 $\pm$ 1.64 & 13.26 $\pm$ 0.65 & 2.33 $\pm$ 0.33 & 2.72 $\pm$ 0.18 & 1.99 $\pm$ 0.13 \\
 & \texttt{Platt} & 15.15 $\pm$ 2.32 & 18.53 $\pm$ 2.08 & 13.57 $\pm$ 0.52 & 2.31 $\pm$ 0.38 & 2.75 $\pm$ 0.25 & 2.04 $\pm$ 0.11 \\
 & \texttt{LOGR} & 15.17 $\pm$ 2.66 & 18.77 $\pm$ 1.28 & 15.10 $\pm$ 0.33 & 2.35 $\pm$ 0.43 & 2.86 $\pm$ 0.17 & 2.32 $\pm$ 0.09 \\
 & \texttt{IGLB} & 18.29 $\pm$ 2.31 & 19.33 $\pm$ 1.61 & 14.63 $\pm$ 1.64 & 3.10 $\pm$ 0.43 & 3.20 $\pm$ 0.21 & 2.42 $\pm$ 0.28 \\
\bottomrule
\end{tabular}}
    \label{tab:placeholder}
\end{table}

\end{document}